\documentclass[11pt]{article}

\usepackage[utf8]{inputenc} 

\pdfoutput=1

\usepackage{algorithm}
\usepackage{algorithmic}
\usepackage{amsmath}
\usepackage{amssymb}
\usepackage{amsthm,amsfonts}
\usepackage{array}
\usepackage{authblk}
\usepackage{babel}
\usepackage{booktabs}
\usepackage{caption}
\usepackage{colortbl}
\usepackage{enumitem}
\usepackage[T1]{fontenc}    
\usepackage{framed}
\usepackage{fullpage}
\usepackage{graphicx}
\usepackage[colorlinks, linkcolor=red, anchorcolor=blue, citecolor=blue]{hyperref}
\usepackage{indentfirst}
\usepackage[utf8]{inputenc}
\usepackage{lipsum}        
\usepackage{makecell}
\usepackage{mathrsfs}
\usepackage{mathtools}

\usepackage[framemethod=default]{mdframed}
\usepackage{microtype}     
\usepackage{multicol}
\usepackage{multirow}
\usepackage{natbib}
\usepackage{nicefrac}       
\usepackage{tablefootnote}
\usepackage{url}            
\usepackage{wrapfig,lipsum}
\usepackage{xcolor}         
\usepackage{xspace}

\usepackage{doi}
\usepackage{tikz}
\usetikzlibrary{decorations.pathreplacing}
\usepackage{amsmath,amsfonts,bm}

\def\1{\bm{1}}

\def\ry{{\textnormal{y}}}

\def\rvx{{\mathbf{x}}}
\def\rvy{{\mathbf{y}}}
\def\rvz{{\mathbf{z}}}

\def\vzero{{\bm{0}}}
\def\vone{{\bm{1}}}

\def\ve{{\bm{e}}}

\def\vq{{\bm{q}}}

\def\vu{{\bm{u}}}

\def\vx{{\bm{x}}}
\def\vy{{\bm{y}}}
\def\vz{{\bm{z}}}

\def\mA{{\bm{A}}}

\def\mI{{\bm{I}}}

\def\mM{{\bm{M}}}

\def\mQ{{\bm{Q}}}
\def\mR{{\bm{R}}}

\DeclareMathAlphabet{\mathsfit}{\encodingdefault}{\sfdefault}{m}{sl}
\SetMathAlphabet{\mathsfit}{bold}{\encodingdefault}{\sfdefault}{bx}{n}

\def\gE{{\mathcal{E}}}

\def\gL{{\mathcal{L}}}

\def\gQ{{\mathcal{Q}}}

\def\gX{{\mathcal{X}}}
\def\gY{{\mathcal{Y}}}

\def\Pr{\mathbb{P}}

\newcommand{\grad}{\ensuremath{\nabla}}

\newcommand{\E}{\mathbb{E}}

\newcommand{\R}{\mathbb{R}}

\newcommand{\KL}[2]{\mathrm{KL}\left(#1 \big\| #2\right)}
\newcommand{\Breg}[2]{D_\mathrm{\phi}\left(#1 \big\| #2\right)}

\newcommand{\TVD}[2]{\mathrm{TV}\left(#1, #2\right)}

\newcommand{\der}{\mathrm{d}}

\makeatletter
\def\thickhline{%
  \noalign{\ifnum0=`}\fi\hrule \@height \thickarrayrulewidth \futurelet
   \reserved@a\@xthickhline}
\def\@xthickhline{\ifx\reserved@a\thickhline
               \vskip\doublerulesep
               \vskip-\thickarrayrulewidth
             \fi
      \ifnum0=`{\fi}}
\makeatother

\newlength{\thickarrayrulewidth}
\theoremstyle{plain}
\newtheorem{theorem}{Theorem}[section]

\newtheorem{lemma}[theorem]{Lemma}

\theoremstyle{definition}

\theoremstyle{remark}
\newtheorem{remark}[theorem]{Remark}

\newcommand{\Cb}[1]{\mathrm{Cube}\left(#1\right)}
\newcommand{\Ce}[1]{\mathrm{Cell}\left(#1\right)}

\NewDocumentCommand{\yl}{ mO{} }{\textcolor{brown}{\textsuperscript{\textit{YL}}\textrm{{\small[#1]}}}}

\newcommand{\ourmethod}{\text{QTD}}

\NewDocumentCommand{\nikki}{ mO{} }{\textcolor{purple}{\textsuperscript{\textit{Nikki}}\textrm{{\small[#1]}}}}

\title{Almost Linear Convergence under Minimal Score Assumptions: Quantized Transition Diffusion}

\author[1]{\normalsize Xunpeng Huang$^*$}
\author[1]{Yingyu Lin$^*$}
\author[1]{Nikki Lijing Kuang}
\author[4]{Hanze Dong}
\author[2]{\\Difan Zou$^\dagger$}
\author[1]{Yian Ma$^\dagger$}
\author[3]{Tong Zhang}

\affil[1]{University of California San Diego}
\affil[2]{The University of Hong Kong}
\affil[3]{University of Illinois Urbana-Champaign}
\affil[4]{SalesForce AI Research}

\begin{document}

\date{}
\maketitle

\def\thefootnote{*}\footnotetext{Equal contribution}
\def\thefootnote{$\dagger$}\footnotetext{Mail to \href{yianma@ucsd.edu}{yianma@ucsd.edu}, \href{dzou@cs.hku.hk}{dzou@cs.hku.hk}}

\begin{abstract}

Continuous diffusion models have demonstrated remarkable performance in data generation across various domains, yet their efficiency remains constrained by two critical limitations: (1) the local adjacency structure of the forward Markov process, which restricts long-range transitions in the data space, and (2) inherent biases introduced during the simulation of time-inhomogeneous reverse denoising processes. To address these challenges, we propose \textbf{Quantized Transition Diffusion} (\ourmethod), a novel approach that integrates data quantization with discrete diffusion dynamics. 
Our method first transforms the continuous data distribution $p_*$ into a discrete one $q_*$ via histogram approximation and binary encoding, enabling efficient representation in a structured discrete latent space. We then design a continuous-time Markov chain (CTMC) with Hamming distance-based transitions as the forward process, which inherently supports long-range movements in the original data space. For reverse-time sampling, we introduce a \textit{truncated uniformization} technique to simulate the reverse CTMC, which can provably provide unbiased generation from $q_*$ under \textit{minimal score assumptions}. Through a novel KL dynamic analysis of the reverse CTMC, we prove that QTD can generate samples with $O(d\ln^2(d/\epsilon))$ score evaluations in expectation to approximate the $d$--dimensional target distribution $p_*$ within an $\epsilon$ error tolerance. 
Our method not only establishes state-of-the-art inference efficiency but also advances the theoretical foundations of diffusion-based generative modeling by unifying discrete and continuous diffusion paradigms.

\end{abstract}

\section{Introduction}
Diffusion models~\cite{sohl2015deep,song2019generative,ho2020denoising} have become a powerful and widely used class of generative models, achieving state-of-the-art (SOTA) performance across diverse domains, including image~\cite{nichol2021improved,rombach2022high,ho2022cascaded}, audio~\cite{schneider2023archisound,kong2020diffwave,popov2021grad}, and video generation~\cite{ho2022video,yang2023diffusion}, as well as scientific discovery~\cite{guo2023diffusion,trippe2023diffusion,watson2023novo,boffi2023probability}. 
The core idea of their design lies in a noising-denoising process: the forward process incrementally adds noise to the training data, mapping an unknown and potentially complex distribution to a simpler prior (often standard Gaussian), while the reverse process progressively denoises samples into the original data distribution by estimating the logarithmic gradient (aka \emph{score}) of the noised distributions~\cite{vincent2011connection, song2019generative}.
Despite their notable empirical successes, understanding and improving the runtime complexity of generating high-quality samples, especially in high-dimensional settings, remain a major challenge.

Various theoretical works~\cite{chensampling,chen2023improved,bentonnearly,li2024d} study continuous diffusion models for generating $d$--dimensional samples (or approximating the training distribution within an $\epsilon$ tolerance) by simulating the time-inhomogeneous reverse Ornstein--Uhlenbeck (OU) process. For instance, DDPM~\cite{ho2020denoising} is proved to achieve an $\tilde{O}(d/\epsilon)$ complexity for total variation (TV) distance convergence under minimal smoothness assumptions~\cite{chen2023improved}. 
Some DDPM variants~\cite{huang2024reverse,li2024provable} improve or balance complexity to the extent of $\tilde{O}(\sqrt{d}/\epsilon)$ or $\tilde{O}(d^{5/4}/\sqrt{\epsilon})$, but require stricter conditions such as smooth score function along the entire OU process.
There are two factors limiting the improvement of the current results.
\textbf{(1) The local adjacency structure} of the forward process: the forward OU process confines each update to a small neighborhood with a high probability.
This neighborhood transition structure constrains the particle movement in each iteration to be tiny, inversely proportional to the expected smoothness of the noised score, so as to control the cumulative error in the inference process.
As a result, using small step sizes hinders the convergence of the particles' distribution to the original data distribution.
\textbf{(2) inherent biases} introduced by discretizing and simulating time-inhomogeneous reverse OU processes: the ideal reverse OU process corresponds to a time-inhomogeneous Markov semigroup governed by the Fokker--Planck equation, yet it cannot be unbiasedly implemented through existing numerical techniques in the diffusion inference pipeline.

In this work, we propose a new \emph{quantized transition diffusion} method, \ourmethod, which addresses the two issues outlined above and attains a total variation (TV) convergence with $O\bigl(d\ln^2(d/\epsilon)\bigr)$ expected score evaluations under minimal score assumptions. 
The core idea is to transform data distribution on the continuous space into a discrete one, which is then parameterized and sampled with a novel discrete diffusion model that we design~\cite{lou2024discrete,zhang2024convergence}. 
For the discrete diffusion model, we first design the structure of the space by leveraging the Hamming distance of binary-encoded states. 
This leads to a sparse graph structure whose diameter and out-degree both grow \emph{logarithmically}, as explained in Fig.~\ref{fig:adj_diff_space} and Sec.~\ref{sec:binary_encode}.  
This design balances the number of jumps required to reach one state from another against the number of options for transition that we need to consider at each node. 
The former is related to the number of iterations required for Markov chain convergence, the latter relates to the complexity of computing the transition probability in each iteration.
Over this discrete space, we design a forward continuous-time Markov Chain (CTMC).
To simulate the reverse process, we design an unbiased simulation technique called \emph{truncated uniformization}, which generalizes classical uniformization methods~\citep{van1992approximate,van2018uniformization} to our setting without additional assumptions.
Our main contributions are summarized as follows.
\begin{itemize}[leftmargin=*]
    \item We propose the \ourmethod\ framework and provably improve the inference rate from {polynomial} to {logarithmic} dependence on $\epsilon$. Specifically, \ourmethod\  generate $d$--dimensional samples to approximate the data distribution with $\Theta(\epsilon)$--TV error with only $O(d\ln^2(d/\epsilon))$ expected score evaluations.

    \item We present a new perspective on modeling continuous data distributions by discretizing the state space and replacing Euclidean ($\ell_2$) neighborhoods with a Hamming-distance-based graph over binary encodings. This allows the discrete process to capture long-range transitions in the original space through sparse, structured jumps in the discrete domain.

    \item We introduce the \emph{truncated uniformization} technique for an unbiased and tractable CTMC simulation. This method removes the restricted bounded-score assumption imposed in prior discrete diffusion analyses~\cite{chen2024convergence,zhang2024convergence}.

    \item We develop a novel proof technique for analyzing the inference process of discrete diffusion models. In place of the standard Girsanov-based approach~\cite{chen2024convergence,zhang2024convergence}, we leverage the chain rule of KL divergence over infinitesimal time intervals to derive convergence guarantees.
\end{itemize}

\begin{figure}[t]
    \centering
    \includegraphics[width=0.7\linewidth]{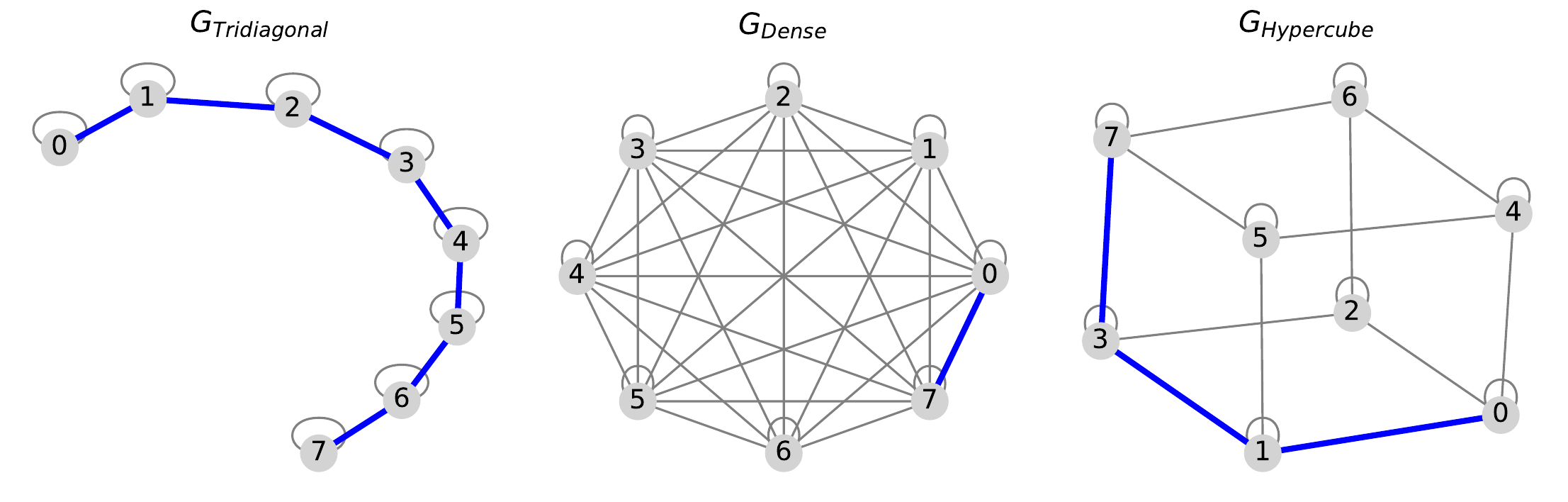}
    \caption{\small Visualization of different adjacency structures. The bold blue edges highlight a diameter path---a shortest path between the two most distant vertices in each graph. Drawing samples in a discrete space $\gY$ by simulating a CTMC can be viewed as traversing a graph whose diameter governs the number of iterations required for convergence, while the out-degree of each node influences the per-iteration complexity. In the neighborhood adjacency $G_{\text{Tridiagonal}}$, each node has an out-degree of $O(1)$ but a diameter of $O(|\gY|)$. For the dense adjacency, the graph $G_{\text{Dense}}$ attains a diameter of $O(1)$ at the cost of an $O(|\gY|)$ out-degree. Notably, the binary adjacency $G_{\text{Hypercube}}$ offers a balanced design, featuring both a diameter and an out-degree of $O(\log|\gY|)$.}
    \label{fig:adj_diff_space}
    \vspace{-0.1in}
\end{figure}

\section{Preliminaries}
\label{sec:pre}

Our goal is to approximate continuous target distributions via tractable discrete processes. In this section, we define discrete forward and reverse Markov processes, parameterized by transition rate functions, and introduce the uniformization technique \cite{van1992approximate,van2018uniformization} to simulate these processes efficiently. All notations introduced below are summarized in Table~\ref{tab:notations} in Appendix~\ref{appendix:notation}.


\noindent\textbf{Problem setup.} Without loss of generality, we focus on distributions that admit probability density functions in Euclidean space.
These continuous density functions are represented by $p\colon \R^d\rightarrow \R^+$.
Specifically, let the data distribution be $p_*\propto\exp(-f_*)$ for some potential function $f_*$.
We consider the task of approximating $p_*$ using some discrete distribution with probability mass function $q_*\colon \gY \rightarrow \R_0^+$, defined on a finite discrete space $\gY$.
This discrete approximation is modeled via a forward Markov process $\{\rvy^\to_t\}_{t=0}^T$ and named as discrete diffusion model~\cite{lou2024discrete,zhang2024convergence, chen2024convergence}, with initial distribution $q_0^\to = q_*$ that evolves toward the uniform distribution.
Then, the marginal distribution at time $t$ is denoted by $q^\to_t$, the joint and conditional distributions over different time steps $t^\prime > t$ are given by
\begin{equation*}
    (\rvy^\to_{t^\prime},\rvy_t^\to)\sim q^\to_{t^\prime, t}\quad \text{and}\quad q^\to_{t^\prime|t}(\vy^\prime|\vy) = q^\to_{t^\prime, t}(\vy^\prime,\vy)/q_t^\to(\vy).
\end{equation*}
For simplicity, we set the forward process to be a time-homogeneous CTMC constructed via the transition rate function $R^\to\colon \gY\times \gY \rightarrow \R$, which implies that both conditional and marginal distributions satisfy
\begin{equation}
    \label{eq:fwd_func}
    \frac{\der q^\to_{t|s}}{\der t}(\vy) =\sum_{\vy^\prime \in \gY} R^\to(\vy,\vy^\prime)\cdot q^\to_{t|s}(\vy^\prime)\quad \text{and}\quad \frac{\der q^\to_t}{\der t}(\vy) =\sum_{\vy^\prime \in \gY} R^\to(\vy,\vy^\prime)\cdot q^\to_t(\vy^\prime).
\end{equation}
The transition rate function $R^\to$ characterizes the instantaneous rate of transitioning from state $\vy^\prime$ to $\vy$ and is formally defined as
\begin{equation}
    \label{def:mean_of_R}
    R^\to(\vy, \vy^\prime) \coloneqq  \lim_{\Delta t\rightarrow 0}\left[\Delta t^{-1}\cdot \left(q^\to_{\Delta t|0}(\vy|\vy^\prime)-\delta_{\vy^\prime}(\vy)\right)\right],
\end{equation}
where $\delta_{\vy^\prime}(\vy) = 1$ if $\vy = \vy^\prime$ and $0$ otherwise.

\noindent\textbf{Reverse process.} Additional properties of $R^\to$ are discussed in Appendix~\ref{sec:a1_fwd_and_infi}. 
To sample from the target distribution $q_* = q_0^\to$ in practice, we simulate the reverse-time process ${\rvy_t^\gets}$ that starts from $q_T^\to$ and moves backward.
\begin{equation*}
    \left\{\rvy^\gets_{t}\right\}_{t=0}^T\quad \text{where}\quad \rvy^\gets_t\sim q^\gets_t = q^\to_{T-t},\quad (\rvy^\gets_{t^\prime},\rvy^\gets_t)\sim q^\gets_{t^\prime, t},\quad \text{and}\quad q^\gets_{t^\prime, t} = q^\gets_t\cdot q^\gets_{t^\prime|t},
\end{equation*}
whose dynamic follows from 
\begin{equation}
    \label{eq:rev_func}
    \frac{\der q^\gets_t}{\der t}(\vy) = \sum_{\vy^\prime\in \gY} R^\gets_t(\vy,\vy^\prime)\cdot q^\gets_t(\vy^\prime)\quad \text{where}\quad R_t^\gets(\vy, \vy^\prime) \coloneqq   R^\to(\vy^\prime, \vy)\cdot\frac{q^\gets_t(\vy)}{q^\gets_t(\vy^\prime)},
\end{equation}
proven in Appendix~\ref{sec:a2_reverse_prop}.
Similar to the $R^\to$ in the forward process, $R^\gets_t$ characterizes the transition rates for the time-inhomogeneous reverse process $\{\rvy^\gets_t\}_{t=0}^T$, i.e., 
\begin{equation}
    \label{def:mean_of_R_gets}
    R^\gets_t(\vy, \vy^\prime) \coloneqq  \lim_{\Delta t\rightarrow 0}\left[\frac{q^\gets_{t+\Delta t|t}(\vy|\vy^\prime)-\delta_{\vy^\prime}(\vy)}{\Delta t}\right],
\end{equation}
as shown in Appendix~\ref{sec:app_prof_mean_of_r_gets}.
In practice, the probability density ratio $q^\gets_t(\vy)/q^\gets_t(\vy^\prime)$ will usually be approximated with neural networks due to its unknown closed form, which is presented as 
\begin{equation*}
    \tilde{v}_{t,\vy^\prime}(\cdot)\approx v_{t,\vy^\prime}(\cdot) = q^\gets_t(\cdot)/q_t(\vy^\prime).
\end{equation*}
To simulate the reverse process in Eq.~\eqref{eq:rev_func}, we must estimate the time-varying rate matrix $R_t^\gets$, which depends on the intractable ratio $q_t^\gets(\vy)/q_t^\gets(\vy')$. We approximate this ratio using neural networks trained via score entropy minimization ~\cite{lou2024discrete,benton2024denoising},
\begin{equation}
    \label{eq:score_estimation}
    L_{\text{SE}}(\hat{v}) = \int_0^T \E_{\rvy_t\sim q^\to_t}\left[\sum_{\vy\not=\rvy_t} R^\to(\rvy_t, \vy)\cdot \Breg{v_{T-t, \rvy_t}(\vy)}{\tilde{v}_{T-t, \rvy_t}(\vy)} \right]\der t,
\end{equation}
where $\Breg{\cdot}{\cdot}$ denotes the Bregman divergence, and $\phi(c)=c\ln c$. 
Similar to the score estimation loss in continuous cases~\cite{chensampling}, the loss $L_{\text{SE}}$ is not directly estimable. Instead, implicit score entropy and denoising score entropy~\cite{lou2024discrete,benton2024denoising} are introduced to enable an equivalent minimization.

\noindent\textbf{Uniformization.} 
With a well-trained score estimation $\tilde{v}_{t}$, uniformization simulates CTMCs by decoupling transition timing from state changes: it samples candidate transition times from a Poisson process with rate $\beta$, and then selects the next state based on a normalized version of the rate matrix. This avoids evaluating transition rates at every fine-grained time step without compromising accuracy.
Specifically, uniformization splits the probability that a state remains unchanged into two scenarios: 
first, no state transition event happens, and second, the state transitions but ultimately returns to itself.
When state self-transition dominates, most steps are spent in place.
Uniformization suggests focusing on the number of actual transitions within a certain interval or on the waiting time until the next transition, hence effectively reducing the frequency of calls to the transition rate evaluations of $R_t^\gets$.
Consider the reverse process presented in Eq.~\eqref{eq:rev_func}, the conditional transition probability satisfies
\begin{equation}
    \label{eq:ideal_infini_ope}
    q^\gets_{t+\Delta t|t}(\vy^\prime|\vy) = \left\{
        \begin{aligned}
            & \Delta t \cdot R^\gets_{t}(\vy^\prime, \vy) && \vy^\prime\not= \vy\\
            & 1-\Delta t\sum_{\tilde{\vy}\not= \vy} R^\gets_t(\tilde{\vy}, \vy) && \vy^\prime = \vy
        \end{aligned}
    \right.
\end{equation}
in an infinitesimal time $\Delta t$ due to Eq.~\eqref{def:mean_of_R_gets}, where the $o(\Delta t)$ term is omitted.
Suppose the probability of transitioning to a different state is upper bounded by $\Delta t \cdot \beta$:
\begin{equation}
    \label{ineq:trans_rate_upb}
    \sum_{\vy^\prime\not=\vy}R^\gets_t(\vy^\prime,\vy) \coloneqq R^\gets_t(\vy) \le \beta, ~~\forall t.
\end{equation}
We can then simulate Eq.~\eqref{eq:rev_func} with the following uniformization procedure:
\begin{enumerate}[leftmargin=*]
    \item With probability $\Delta t\cdot \beta$, allow a state transition.
    \item Conditioning on an allowed transition, move from $\vy$ to $\vy^\prime$ with probability
    \begin{equation*}
        \mM_t(\vy^\prime|\vy) = \left\{
            \begin{aligned}
                & \beta^{-1}R_t^\gets(\vy^\prime, \vy) && \vy^\prime\not=\vy\\
                & 1- \beta^{-1}R^\gets_t(\vy) && \text{otherwise}
            \end{aligned}
        \right. .
    \end{equation*}
\end{enumerate}
Under these two steps, the practical conditional probability satisfies 
\begin{equation}
    \label{eq:prac_infini_ope}
    \hat{q}_{t+\Delta t| t}(\vy^\prime|\vy) = \left\{
        \begin{aligned}
        & \Delta t\cdot \beta \cdot R_t^\gets(\vy^\prime, \vy)\cdot \beta^{-1} = \Delta t \cdot R_t^\gets(\vy^\prime,\vy) && \vy^\prime\not=\vy\\
        & 1-\Delta t\cdot \beta + \Delta t \cdot \beta \cdot (1-\beta^{-1}\cdot R^\gets_t(\vy)) = 1-\Delta t\cdot R_t^\gets(\vy) && \vy^\prime = \vy,
        \end{aligned}    
    \right.
\end{equation}
which exactly matches Eq.~\eqref{eq:ideal_infini_ope}.
Under this condition, the number of transition events within a time interval $[s, t]$ follows a Poisson distribution~\cite{van1992approximate,van2018uniformization} whose expectation is $\beta(t-s)$, which coincides with the number of required evaluations of the transition rate function $R^\gets_t$.
This implies choosing a tighter upper bound $\beta$ directly leads to better complexity.


\section{Quantized Transition Diffusion}
\label{sec:method}

In this section, we present a novel Quantized Transition Diffusion (\ourmethod) for efficiently approximating samples from a continuous data distribution. Our approach addresses the inefficiency of standard diffusion-based inference in continuous space by discretizing the problem into a structured CTMC over a binary-encoded state space. Key innovations include (i) a histogram-based approximation of the target density, (ii) a binary embedding that enables long-range transitions while maintaining manageable state connectivity, and (iii) a truncated uniformization scheme for efficient and unbiased simulation of the reverse-time CTMC. 

\subsection{Histogram Approximation}

\begin{figure}
    \centering
    \includegraphics[width=0.7\linewidth]{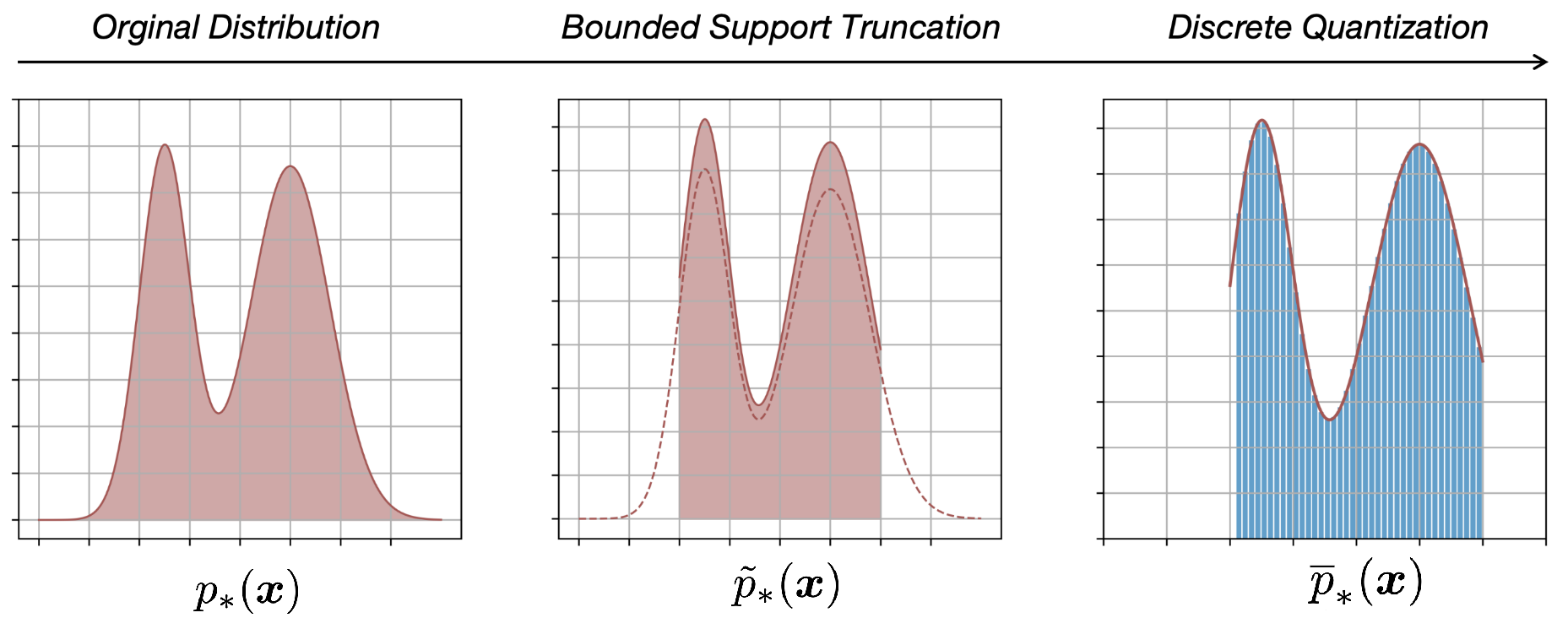}
    \caption{\small Visualization of the histogram approximation. The first step regularizes the original distribution in some bounded sets but controls the TV gap by Lemma~\ref{lem:histog_dis_comp}. The second step quantizes the probability density to a histogram-like distribution but controls the TV gap by Lemma~\ref{lem:histog_dis_quan}.}
    \label{fig:DisPro_1}
\end{figure}

To approximate the target distribution $p_*$ defined in the Euclidean space $\R^d$ with a histogram-like distribution, we first restrict its support to a bounded region, which can be represented by a cube of side length $L$ as follows:
\begin{equation*}
    \Cb{L}\coloneqq \left\{\vx \,\middle|\,-L\le \vx_j \le L,\ \forall j\in\{1,2,\ldots,d\}\right\}. 
\end{equation*}

Given that $\Cb{L}$ covers most probability mass of $p_*$, we construct a probability density restricted to this region to approximate $p_*$:
\begin{equation}
    \label{def:bounded_approx}
    \tilde{p}_*(\vx)\coloneqq \frac{p_*(\vx)}{\int_{\vx\in \Cb{L}}p_*(\vx)\,\der \vx}\quad \forall \vx\in\Cb{L}.
\end{equation}
Standard concentration arguments allow us to control the TV distance between $p_*$ and $\tilde{p}_*$. Next, we quantize $\tilde{p}_*$ over $\Cb{L}$ by discretizing each dimension into $K\coloneqq 2L/l$ intervals of width $l$, with partition points defined by:
\begin{equation*}
    l_{i} =-L + i\cdot l\quad i\in\{0, 1,\ldots, K\}\quad \text{and}\quad -L\le l_i\le L.
\end{equation*}
That means the high-dimensional cube $\Cb{L}$ will be decomposed into $K^d$ cells (subsets), and each cell will cover a small region shown as follows
\begin{equation}
    \label{def:cell}
    \Ce{i_0, i_1,\ldots,i_{d-1}} \coloneqq \left\{\vx|l_{i_j} < \vx_j \le l_{i_j+1}, \forall j\in\{0,1,\ldots, d-1\} \right\}.
\end{equation}
We construct the piecewise constant distribution $\overline{p}_*(\vx)$ by averaging the original density $\tilde{p}_*$ over each quantization cell. Specifically, for each cell $\Ce{i_0, i_1,\ldots,i_{d-1}}$, we assign a constant density to all points $\vx$ in the cell:
\begin{equation}
    \label{def:quantize_cell}
    \overline{p}_*(\vx)=l^{-d}\cdot \int_{\vu \in \Ce{i_0, i_1,\ldots,i_{d-1}}}\tilde{p}_*(\vu)\der \vu, \quad \vx\in\Ce{i_0, i_1,\ldots,i_{d-1}}.
\end{equation}
This construction ensures that $\int_{\Cb{L}} \overline{p}_*(\vx) \der \vx = 1$.
As shown in the following lemma, under the smoothness assumption on $p_*$, we can control the TV distance between $\overline{p}_*$ and ${p}_*$. It implies that with proper choices of $L$ and $l$, the histogram-like distribution $\overline{p}_*$ can be made arbitrarily close to $p_*$.
\begin{lemma}
    \label{lem:discrete_quantization_gap}
    Suppose the target distribution $p_*\propto \exp(-f_*)$ is $\sigma$ sub-Gaussian and $f_*$ is $H$--smooth, we can construct $\overline{p}_*$ defined on a finite cube $\Cb{L}$ with length
    \begin{equation*}
        L = \sigma\cdot \sqrt{2\ln (2d/\epsilon)} \quad \mathrm{and}\quad  l= \left[2H\cdot \left(\sigma\sqrt{2d\ln(2d/\epsilon)} + d + \sqrt{dm_0}\right)\right]^{-1} \cdot \epsilon,
    \end{equation*}
    to satisfy $\TVD{p_*}{\overline{p}_*}\le 3\epsilon$.
\end{lemma}
We defer the proof to Appendix~\ref{sec:discrete_quantization_gap}.
Under this condition, we have constructed a histogram-like distribution to approximate $p_*$, which can be visualized by Fig.~\ref{fig:DisPro_1}.


\subsection{Binary Encoding of the Discrete Space}
\label{sec:binary_encode}
While direct discretization via grid quantization is natural, it suffers from exponentially increasing connectivity, which increases the complexity of transition probability calculation. To address this, we introduce a binary encoding scheme that allows efficient long-distance transitions in Euclidean space with only $\mathcal{O}(d \log K)$ neighbors per state.
Recall from Eq.~\eqref{def:quantize_cell} that distribution $\overline{p}_*$ remains defined on $\mathbb{R}^d$. However, due to the histogram shape, it can be sampled by introducing a discrete distribution $\overline{q}_*$, which is defined as
\begin{equation}
    \label{def:quantized_mass_func}
    \overline{q}_*(\vy) \propto \overline{p}_*(-L\cdot \vone + l\cdot (\vy - 0.5\cdot \vone)),\quad \text{where}\quad \vy\in\{0,1,\ldots, K-1\}^d.
\end{equation}
This means that we integrate all points of each cell into a discrete state whose probability mass function is proportional to the probability density at the midpoint of $\Ce{\vy}$.
Consequently, sampling from $\overline{p}$ reduces to the following two-stage procedure:
\begin{enumerate}
    \item Sample from the discrete distribution $\overline{q}_*$ defined on $\overline{\gY}$;
    \item Uniformly draw a sample from the cell, i.e., $\Ce{\vy}$.
\end{enumerate}
Then, we can obtain samples from $\overline{p}_*$ that are arbitrarily close to $p_*$.
From the diffusion modeling perspective, the remaining challenge is how to parameterize $\overline{q}_*$.

In continuous diffusion models, the score function is typically modeled via a neural network trained to estimate gradients of the log density under noised distributions. Importantly, both the forward noise process and the reverse inference process are governed by the adjacency structure of the particle space.
Specifically, consider an Ornstein--Uhlenbeck (OU) process starting from $p^\to_0 = p_*$. The forward transition kernel is
\begin{equation*}
    p^\to_{t^\prime|t}(\cdot|\vx) = \mathcal{N}\left(e^{-(t^\prime - t)}\cdot \vx , 1-e^{-2(t^\prime - t)}\right).
\end{equation*}
This implies that over an infinitesimal time, the particle $\rvx_t = \vx$ will, with high probability, move to a nearby point $\vx^\prime$ with a small $\|\vx^\prime-\vx\|_2$. Thus, the $L_2$ metric defines the natural adjacency structure in the Euclidean space.
Moreover, this adjacency structure also governs the diffusion inference process. In the reverse OU process, the range of possible next states for a particle is constrained by which states could have transitioned into that particle in the forward OU process.
Because particles are most likely to move to states that are close in terms of the $L_2$ norm, it becomes difficult for the reverse process to make long-range jumps within an infinitesimal time.

However, for a discrete space, such as $\overline{\gY}$, we can use the Hamming distance, i.e.,
\begin{equation*}
    \mathrm{Ham}(\vy,\vy^\prime) = \left|\left\{i | \vy_i\not=\vy^\prime_i\right\}\right|
\end{equation*}
to describe the adjacency structure. 
Two states, e.g., $\vy, \vy^\prime\in\overline{\gY}$, are considered as adjacent only when $\mathrm{Ham}(\vy,\vy^\prime)\le 1$.
Under this condition, we are able to jump from $\vy=(0,\ldots, 0)$ to $\vy^\prime = (K-1, 0, \ldots, 0,)$ in a single step.
When mapped back to Euclidean space, such a jump allows the particle to transit from $(L, -L,\ldots, -L)$ to $(-L, -L,\ldots, -L)$, traversing an entire edge of the cube $\Cb{L}$.  While such long-range jumps are permitted, each discrete state in $\overline{\gY}$ has $\mathcal{O}(d \cdot 2K)$ neighbors, which undermines the sampling efficiency in the reverse process.

To trade off the jump distance and the number of out-degrees of discrete states, we propose a binary encoding scheme for the discrete states in each dimension.
Specifically, assuming $\log_2 K$ is an integer (without loss of generality), we encode the $d$--dimensional state 
\begin{equation*}
    \overline{\vy} = [\overline{\vy}_0, \overline{\vy}_1, \ldots \overline{\vy}_{d-1}] \in \overline{\gY}
\end{equation*}
into $\vy\in \gY \coloneqq \{0,1\}^{d\log_2K}$ by the following one-one mapping
\begin{equation*}
    \vy = [\vy_0, \vy_1, \ldots \vy_{d\log_2K-1}]\quad \text{where}\quad \vy_i = \lfloor \overline{\vy}_{\lfloor i/\log_2K \rfloor} / 2^{i - {\lfloor i/\log_2K \rfloor} } \rfloor\  \mathrm{mod}\  2,
\end{equation*}
and abbreviate this mapping as $\mathrm{vBin}\colon \overline{\gY}\rightarrow \gY$.
With this encoding, drawing samples from $\overline{q}_*$ on $\overline{\gY}$ is equivalent to sampling from $q_*$ on $\gY$, where $q_*(\mathrm{vBin}(\vy)) \coloneqq \overline{q}_*(\vy)$.
We then impose an adjacency structure on $\gY$ using Hamming distance, and require $\mathrm{Ham}(\vy,\vy^\prime)\le 1$ for states $\vy,\vy^\prime\in \gY$.
Then the number of jumps between the following two discrete states:
\begin{equation*}
    \vy = [\underbrace{0, 0,\ldots,0}_{d\log_2K}]\quad \text{and}\quad \vy^\prime = [\underbrace{1,1,\ldots,1}_{\log_2K},0,\ldots, 0]
\end{equation*}
will be $\log_2K$ only. 
When mapped back to Euclidean space, this again corresponds to a transition from $(L, -L, \ldots, -L)$ to $(-L, -L, \ldots, -L)$—a long-range jump, but now each binary state $\vy \in \gY$ has only $d\log_2 K$ adjacent states, offering a dramatic reduction in connectivity compared to the original discrete grid.
We visualize the differences among adjacency structures in Fig.~\ref{fig:adj_diff_space}.

\begin{algorithm}[t]
    \caption{\sc Training Data Quantization}
    \label{alg:data_quanta}
    \begin{algorithmic}[1]
            \STATE {\bfseries Input:}  The training set $\gX=\{\vx^{(i)}\}_{i=1}^N$.
            \STATE Initialize output set $\gY=\{\}$ and the parameters, e.g., $L$ and $l$ as shown in Lemma~\ref{lem:discrete_quantization_gap}.
            \FOR{$n = 1$ {\bfseries to} $N$}
                \STATE Quantize the training sample $\vx^{(n)}$ to $\overline{\vy}^{(n)}$ via
                \begin{equation*}
                    \overline{\vy}^{(n)} = [\overline{\vy}^{(n)}_0, \overline{\vy}^{(n)}_1,\ldots, \overline{\vy}^{(n)}_{d-1}] \quad \text{where}\quad \overline{\vy}_i^{(n)} = \lfloor (\vx_i^{(n)}+L)/l \rfloor.
                \end{equation*}
                \STATE Append the set $\gY$ with binary encoded $\vy^{(n)} = \mathrm{vBin}(\overline{\vy}^{(n)})$ where $\vy^{(n)}\in \{0,1\}^{d\log_2K}$.
            \ENDFOR
            \STATE {\bfseries return} $\gY$.
    \end{algorithmic}
\end{algorithm}

\paragraph{The forward CTMC starting from $q_*$.}
Analogous to the role of the Ornstein–Uhlenbeck (OU) process in Euclidean space—which gradually injects noise into $p_*$ to reach a tractable distribution, we aim to design a forward process that transforms $q_*$ into an easy-to-sample distribution, while fully exploiting the balance between long-range transitions and controlled neighborhood size provided by the binary encoding scheme.
In practice, we begin by generating discrete samples $\rvy\sim q_*$ corresponding to data samples $\rvx\sim p_*$ from the training set.
The specific algorithm is given in Alg.~\ref{alg:data_quanta}.
Once $q_*$ is established, we follow from Eq.~\eqref{eq:fwd_func} to construct transition rate function as
\begin{equation}
    \label{eq:fwd_transtion_rate_max}
    R^\to(\vy, \vy^\prime) = \left\{
        \begin{aligned}
            & 1 && \mathrm{Ham}(\vy,\vy^\prime) = 1\\
            & -d\log_2K && \vy=\vy^\prime\\
            & 0 && \text{otherwise}
        \end{aligned}.
    \right. 
\end{equation}
This choice defines a simple and symmetric CTMC where each state has exactly $d\log_2K$ neighbors, each reachable at a unit rate. As a result, the forward process behaves like a time-homogeneous diffusion over the hypercube, and converges linearly to the stationary distribution $q^\to_\infty$ following Lemma~\ref{lem:fwd_convergence}. The proof is deferred to Appendix~\ref{sec:app_dis_fwd_conv}.
\begin{lemma}
    \label{lem:fwd_convergence}
    Suppose the transition rate function $R^\to$ of CTMC $\{\rvy^\to_t\}_{t=0}^{T}$ is set as Eq.~\eqref{eq:fwd_transtion_rate_max}, the underlying distribution $q^\to_t$ of $\rvy^\to_t$ satisfies
    \begin{equation*}
        \KL{q^\to_t}{q^\to_\infty}\le e^{-t}\cdot d\log_2K.
    \end{equation*}
\end{lemma}

\subsection{Truncated Uniformization}
\label{sec:truncted_unif}

As shown in Section~\ref{sec:pre}, the complexity of simulating a CTMC via uniformization is closely tied to the upper bound on the total transition rate to states other than the current one—denoted by 
$\beta$ in Eq.~\eqref{ineq:trans_rate_upb}. A tighter upper bound on this rate improves efficiency.
This observation motivates us to explore the time-dependent $\beta_t$ for the time-inhomogeneous reverse CTMC given by Eq.~\eqref{eq:rev_func}. 
Specifically, if we choose the transition rate function as in Eq.~\eqref{eq:fwd_transtion_rate_max} for the forward CTMC in Eq.~\eqref{eq:rev_func}, the resulting time-varying upper bound $\beta_t$ satisfies the lemma below. The proof is deferred to Appendix~\ref{sec:out_degree_rate_wrt_time}.
\begin{lemma}
    \label{lem:out_degree_rate_wrt_time}
    Consider a CTMC whose transition rate function $R^\to$ is defined as Eq.~\eqref{eq:fwd_transtion_rate_max}. Then, for any $\vy$, the reverse transition rate function satisfies
    \begin{equation}
        \label{def:beta_t_}
        \sum_{\vy^\prime\not=\vy}R^\gets_t(\vy^\prime,\vy) \coloneqq R^\gets_t(\vy)\le \beta_t\coloneqq  (2d\log_2K)\cdot \max\{1, (T-t)^{-1}\}.
    \end{equation}
\end{lemma}
Therefore, it is important for us to divide the entire reverse process into $W$ segments. 
With a proper segmentation, we can assign a tight upper bound $\beta_{t_w}$ for $R^\gets_t(\vy)$ when $t\in[t_{w-1},t_{w})$ and minimize the expectation of transition events, given by $\sum_w \beta_{t_w}\cdot (t_{w}-t_{w-1})$.
In practice, since the exact form $R^\gets_t$ is intractable, we approximate it by minimizing Eq.~\eqref{eq:score_estimation}:
\begin{equation*}
    R^\gets_t(\vy,\vy^\prime)\approx \tilde{R}_{t}(\vy,\vy^\prime) = R^\to(\vy^\prime, \vy)\cdot \tilde{v}_{t,\vy^\prime}(\vy).
\end{equation*}
Here, $\tilde{v}_{t,\vy^\prime}(\vy)$ can approximate the ideal density ratio $q^\gets_t(\vy)/q^\gets_t(\vy^\prime)$ with high accuracy.
However, this approximation may violate the desired global rate bound in Lemma~\ref{lem:out_degree_rate_wrt_time}. To address this, prior work~\cite{chen2024convergence} imposes an estimated score boundedness assumption for discrete diffusion inference:
\begin{equation}
    \label{ineq:estimated_score_bound}
    \sum_{\vy\not=\vy^\prime}\tilde{R}_t(\vy,\vy^\prime)\le Cd\log_2K\cdot \max\{1, (T-t)^{-1}\}
\end{equation}
We argue that this assumption can be safely removed by truncating the approximate transition rate function as follows:
\begin{equation}
    \label{def:prac_infi_oper_1}
    \hat{R}_t(\vy,\vy^\prime) = \left\{
        \begin{aligned}
            & \tilde{R}_t(\vy, \vy^\prime)\cdot \beta_t/\tilde{R}_t(\vy^\prime) && \tilde{R}_t(\vy^\prime)>\beta_t\\
            & \tilde{R}_t(\vy, \vy^\prime) && \text{otherwise}.
        \end{aligned}
    \right., \quad \forall \vy^\prime\not=\vy,
\end{equation}
and
\begin{equation}
    \label{def:prac_infi_oper_2}
    \hat{R}_t(\vy^\prime,\vy^\prime) = -\sum_{\vy\not=\vy^\prime}\hat{R}_t(\vy,\vy^\prime).
\end{equation}
It ensures that the total outgoing rate from any state does not exceed $\beta_t$, hence eliminating the need for explicit score bounds. Combining $\hat{R}_t$ with the two-step uniformization mentioned in Section~\ref{sec:pre}, we obtain a practical and efficient inference algorithm, summarized in  Alg.~\ref{alg:uni_inf}. Here, $\ve_i$ denotes the one-hot vector with a $1$ at position $i$ and $0$ elsewhere, and $\mathrm{mod}$ is an element-wise operator.

\begin{algorithm}[t]
    \caption{\sc Inference Process with Truncated Uniformization}
    \label{alg:uni_inf}
    \begin{algorithmic}[1]
            \STATE {\bfseries Input:}  Total time $T$, a time partition $0=t_0<\ldots<t_W=T-\delta$, parameters $\beta_{t_1}, \ldots, \beta_{t_W}$ set as Eq.~\eqref{def:beta_t_}, a reverse transition rate function $\hat{R}^\gets_t$ obtained by the learnt score function $\tilde{v}_{t,\vy^\prime}(\cdot)$.
            \STATE Draw an initial sample $\hat{\rvy}_{t_0} \sim \mathrm{Uniform}(\{0,1\}^{d\log_2K})$. \label{step:infer_with_trunc_uni_start}
            \FOR{$w = 1$ {\bfseries to} $W$}
                \STATE Draw $N\sim \mathrm{Poisson}(\beta_{t_w} (t_w-t_{w-1}))$;
                \STATE Sample $N$ points i.i.d. uniformly from $[t_{w-1}, t_w]$ and sort them as $\tau_1<\tau_2<\ldots<\tau_N$;
                \STATE Set $\rvz_0 = \hat{\rvy}_{t_{w-1}}$;
                \FOR{$n = 1$ {\bfseries to} $N$}
                    \STATE Set 
                    \begin{equation*}
                        \rvz_{n} = \left\{
                            \begin{aligned}
                                & (\rvz_{n-1} + \ve_i) \ \mathrm{mod}\ 2, && w.p.\ \beta_{t_w}^{-1}\cdot \hat{R}^\gets_{\tau_n}(\vz_{n-1}+\ve_i, \vz_{n-1}),\quad 0\le i\le d\log_2K -1\\
                                & \rvz_{n-1}, && w.p.\ 1- \beta_{t_w}^{-1}\cdot \hat{R}^\gets_{\tau_n}(\vz_{n-1}).
                            \end{aligned}
                        \right.
                    \end{equation*}
                \ENDFOR
                \STATE Set $\hat{\rvy}_{t_w} = \rvz_{N}$. \label{step:approximate_discrte_sample}
            \ENDFOR
            \STATE Recover the cell index with $\overline{\rvy}=\mathrm{vBin}^{-1}(\hat{\rvy}_{t_W})$ and uniformly draw a sample $\hat{\rvx}$ from $\Ce{\overline{\rvy}}$.\label{step:from_discrete_to_continuous}
            \STATE {\bfseries return} $\hat{\rvy}_{t_W}$.
    \end{algorithmic}
\end{algorithm}

\section{Theoretical Results}
In this section, we begin by introducing a set of commonly used assumptions for analyzing the inference efficiency of diffusion models. Next, we show that the total variation (TV) distance between the generated and target data distributions decays exponentially under Alg.~\ref{alg:uni_inf}. Finally, we compare the proposed truncated uniformization scheme with alternative discrete inference algorithms, demonstrating its significant advantages.

\noindent\textbf{General Assumptions.}
To analyze convergence and the gradient complexity required to achieve TV distance convergence, we make the following assumptions on $p_*$:
\begin{enumerate}[label=\textbf{[A{\arabic*}]}]
    \item \label{a1} The second moment of $p_*$ is bounded, i.e., $\E_{\rvx\sim p_*}[\|\rvx\|^2]\le m_0$.

    \item \label{a2} The energy function of $p_*$ has bounded Hessian, i.e., $\|\grad^2 \ln p_*\|\le H$.

    \item \label{a3} For any $\vu\in\R^d$, there is a scalar sub-Gaussian tail, i.e.,
    \begin{equation*}
        \E_{\rvx\sim p_*}\left[\exp\left(t\cdot \rvx^\top \vu\right)\right]\le \exp\left(\sigma^2 t^2 \|\vu\|^2/2\right).
    \end{equation*}

    \item \label{a4} Quantize the continuous training set $\gX$ into a discrete one $\gY$ by Alg.~\ref{alg:data_quanta}, and train the discrete score $\tilde{v}_t$ by Eq.~\eqref{eq:score_estimation}, the score estimation error is sufficiently small, i.e., $L_{\text{SE}}(\hat{v})\le \epsilon^2_{\text{score}}$.
\end{enumerate}
Assumptions~\ref{a1} and~\ref{a2} constitute the minimal smoothness conditions proposed in~\cite{chen2023improved}. As noted, Assumption~\ref{a2} can often be circumvented using early stopping trick ~\cite{chen2023improved,bentonnearly}.
It also appears in state-of-the-art convergence analyses such as~\cite{li2024d}.
Although our analysis additionally calls for a light-tailed assumption, it does not impose isoperimetric constraints, and $p_*$ need not be log-concave or unimodal.
Under this condition, Assumption~\ref{a3}, the $\sigma$ sub-Gaussian property, is introduced solely for providing clear convergence. A similar result can be achieved by any distribution with an exponential tail.
Assumption~\ref{a4} is a standard assumption widely used in recent works~\cite{zhang2024convergence,chen2024convergence} to study discrete score estimation error.
Crucially, our analysis does not impose any smoothness or boundedness assumption on the intermediate estimated scores $\tilde{v}_t$. 
We argue that our analysis achieves the minimal score assumption.

\begin{table*}[t]
    \centering
    \caption{\small Comparison with prior works simulating reverse particle SDEs, where \textcolor{red}{\bf[A4]'} denotes the score estimation error trained in Euclidean space and \textcolor{red}{\bf smooth score} denotes the smooth score assumption for the whole OU process ($p_t$) starting from $p_*$. Note that Assumptions~\ref{a2} is only about $p_*$ and can be replaced by the early stopping trick. All complexities for TV convergence are achieved by assuming $\epsilon_{\text{score}} = \tilde{o}(\epsilon)$.} 
    \begin{tabular}{ccccc}
    \toprule
     Results & Algorithm & Assumptions &  Complexity (for TV) \\
     \midrule
    \cite{chensampling} & DDPM & {\ref{a1}},  
    \textcolor{red}{\bf smooth score}, \textcolor{red}{\bf[A4]'} 
    & $\tilde{\mathcal{O}}(d\epsilon^{-2})$\\
     \midrule
     \cite{chen2023improved} & DDPM & \ref{a1}, \ref{a2}, \textcolor{red}{\bf[A4]'}  & $\tilde{\mathcal{O}}(d^2\epsilon^{-2})$\\
     \midrule
     \cite{bentonnearly} & DDPM & \ref{a1}, \ref{a2}, \textcolor{red}{\bf[A4]'}  & $\tilde{\mathcal{O}}(d\epsilon^{-2})$\\
     \midrule
     \cite{li2024d} & DDPM & \ref{a1}, \ref{a2}, \textcolor{red}{\bf[A4]'}  & $\tilde{\mathcal{O}}(d\epsilon^{-1})$\\
     \midrule
     \cite{huang2024reverse} & RTK-ULD & \ref{a1},  
     \textcolor{red}{\bf smooth score}, \textcolor{red}{\bf[A4]'} 
     & $\tilde{\mathcal{O}}(d^{1/2}\epsilon^{-1})$\\
     \midrule
     \cite{li2024provable} & MidPoint-DDPM & \ref{a1}, \ref{a2}, \textcolor{red}{\bf[A4]'}  & $\tilde{\mathcal{O}}(d^{5/4}\epsilon^{-1/2})$\\
     \midrule
     This paper & \ourmethod & 
     \ref{a1}, \ref{a2}, \ref{a3}, \ref{a4},
 & \textcolor{red}{$\tilde{\mathcal{O}}(d)$}\\
     \bottomrule 
    \end{tabular}
    \label{tab:comp_old}
\end{table*}
Under these assumptions, we establish the following theorem, with the proof deferred to Appendix~\ref{sec:proof_main_thm}.
\begin{theorem}
    \label{thm:main_thm}
    Suppose Assumption~\ref{a1}--\ref{a4} hold, if we introduce Alg.~\ref{alg:data_quanta} with
    \begin{equation*}
        L = \sigma\cdot \sqrt{2\ln (2d/\epsilon)} \quad  l= \left[2H\cdot \left(\sigma\sqrt{2d\ln(2d/\epsilon)} + d + \sqrt{dm_0}\right)\right]^{-1} \cdot \epsilon\quad \text{and}\quad K=2L/l.
    \end{equation*}
    to quantize $p_*$, train a discrete diffusion model to satisfy $
        \epsilon_{\text{score}} \le \frac{\epsilon}{\ln(d/\epsilon) + \ln\log_2 K} = \tilde{O}(\epsilon)$,  and implement 
    Alg.~\ref{alg:uni_inf} with $
        t_0 = 0, \quad  t_{w+1} - t_w = 0.5\cdot (T-t_{w+1}),\quad t_W= T-\delta$,
    and $
        \beta_{t_w}\coloneqq 2d \log_2K/{\min\{1, T-t_w\}}$,
    where
    \begin{equation*}
        T = \ln(d/\epsilon) + \ln\log_2K \quad \text{and}\quad \delta\le d^{-1}\epsilon\cdot [\log_2K]^{-1}
    \end{equation*}    
    the expectation of iteration/score estimation complexity of Alg.~\ref{alg:uni_inf} will be $O(d\ln^2(d/\epsilon))$ to achieve $\TVD{p_*}{\hat{p}}\le 5\epsilon$ where $\hat{p}$ denotes the underlying distribution of generated samples.
\end{theorem}
We provide a complexity comparison in Table~\ref{tab:comp_old}. Unlike conventional diffusion models that directly apply the noising--denoising procedure in Euclidean space, \ourmethod\ achieves a SOTA linear convergence rate with respect to the error tolerance $\epsilon$, only requiring the additional mild sub-Gaussian assumption~\ref{a3}. 
Even under more restrictive settings, such as assuming bounded support for the target distribution, prior works for DDPM ~\cite{chen2023improved,chensampling} achieve complexity results that are only comparable to the minimal smooth case presented in Table~\ref{tab:comp_old}.

Moreover, the proposed truncated uniformization technique is of independent interest as a general-purpose inference algorithm for discrete diffusion models.
In comparison to biased discrete inference, such as the Euler method~\cite{zhang2024convergence} and $\tau$-leaping~\cite{ren2025fast}, which respectively require $\tilde{O}(d^{4/3}\epsilon^{-4/3})$ and $\tilde{O}(d\epsilon^{-1})$ complexity to ensure total variation convergence, truncated uniformization method only requires $\tilde{O}(d)$ discrete score evaluations, significantly improving efficiency. 
Further distinguishing itself from standard uniformization methods, truncated uniformization removes the widely-adopted assumption in Eq.~\eqref{ineq:estimated_score_bound}, thus significantly enhancing its practical applicability. 
We defer the comparison table to Table~\ref{tab:comp_old_dis}.
\section{Conclusion and Limitation}
In conclusion, we introduce a novel approach, \ourmethod, which first quantizes the continuous data distribution into a discrete counterpart, and then applies a truncated uniformization procedure to achieve unbiased inference with improved score-evaluation complexity for continuous data generation. 
Beyond its SOTA theoretical complexity---namely, linear convergence with respect to the error tolerance---the truncated uniformization framework is of independent interest as an inference algorithm for discrete diffusion models, where it also attains top-tier theoretical complexity under minimal assumptions.

A key limitation of our approach is that achieving accelerated convergence without degrading generation quality requires the discrete score estimation error to be on par with the continuous score estimation error outlined by \cite{chensampling,bentonnearly}. While some works \cite{meng2022concrete,lou2024discrete} have introduced discrete training objectives such as concrete score matching and denoising score entropy, no direct comparison between discrete and continuous score training has been conducted. Lastly, our study is primarily theoretical, so its scalability and applicability remain to be investigated in real-world settings.

\bibliographystyle{apalike}
\bibliography{0_contents/ref}  





\newpage
\appendix

\section{Notation Summary}
\label{appendix:notation}
We summarize all notations used in the main paper and appendix in Table~\ref{tab:notations}.
\begin{table}[H]
\small
\centering
\caption{\small Summary of key notations used in the paper.}
\vspace{-2pt}
\label{tab:notations}
\renewcommand{\arraystretch}{1.11}
\begin{tabular}{@{}lp{0.75\linewidth}@{}}
\toprule
\textbf{Symbol} & \textbf{Description} \\
\midrule
$\Cb{L}$ & Bounded cube $[-L, L]^d$ covering high-probability mass of $p_*$ \\
$\Ce{i_0, \ldots, i_{d-1}}$ & Quantization cell (hypercubes) defined by coordinate bins, Eq.~\eqref{def:cell} \\
$\gY$ & Binary discrete space $\{0,1\}^{d\log_2 K}$ \\
$\overline{\gY}$ & Grid index space $\{0, \ldots, K-1\}^d$ \\
$\mathrm{vBin}(\cdot)$ & Mapping from grid index $\overline{\gY}$ to binary code $\gY$ \\
\midrule
$p_* \propto \exp(-f_*)$ & Target continuous distribution in $\mathbb{R}^d$ \\
$\tilde{p}_*$ & Truncated and renormalized version of $p_*$ over $\Cb{L}$, Eq.~\eqref{def:bounded_approx}  \\
$\overline{p}_*$ & Histogram approximation to $\tilde{p}_*$ over $\Cb{L}$, Eq.~\eqref{def:quantize_cell} \\
\midrule
$\overline{q}_*$ & Discrete distribution on $\overline{\gY}=\{0, \ldots, K-1\}^d$ induced by $\overline{p}_*$, Eq.~\eqref{def:quantized_mass_func} \\
\midrule
$q_*$ & Discrete distribution on $\gY=\{0,1\}^{d\log_2 K}$, $q_* = \overline{q}_* \circ \mathrm{vBin}^{-1}$  \\
$\rvy^\to_t$ & Forward-time CTMC on $\gY$ \\
$q^\to_t$ & Marginal distribution of forward process at time $t$, i.e., $\rvy^\to_t\sim q^\to_t$ \\
$q^\to_{t',t}$ & Joint distribution of $(\rvy^\to_{t^\prime},\rvy^\to_t)$ \\
$q^\to_\infty$ & Stationary distribution of the forward CTMC (uniform distribution)\\
$q^\to_{t'|t}(\vy'|\vy)$ & Conditional transition probability in forward process, Eq.~\eqref{eq:fwd_func} \\
$\rvy^\gets_t$ & Reverse-time CTMC defined by $q^\gets_t := q^\to_{T-t}$, $\rvy^\gets_t\sim q_t^\gets$ \\
$q^\gets_t$ & Marginal distribution of reverse process at time $t$, $q^\gets_t=q^\to_{T-t}$ \\
$q^\gets_{t',t}$ & Joint distribution of $(\rvy^\gets_{t^\prime},\rvy^\gets_t)$\\
$q^\gets_{t'|t}(\vy'|\vy)$ & Conditional transition probability of the ideal reverse process \\
$\hat{q}_{t+\Delta t|t}(\vy'|\vy)$ & Practical reverse conditional probability, Eq.~\eqref{eq:prac_infini_ope} \\
\midrule
$R^\to(\vy,\vy')$ & Forward transition rate from state $\vy'$ to $\vy$, Eq.~\eqref{def:mean_of_R}, and Eq.~\eqref{eq:fwd_transtion_rate_max}. This follows the ordering of the conditional distribution $p(\vy | \vy')$, which is the \textit{transpose} of the convention used in some other works. \\
$R^\gets_t(\vy,\vy')$ & Reverse transition rate at time $t$ from state $\vy'$ to $\vy$, $R_t^\gets(\vy, \vy^\prime) \coloneqq   R^\to(\vy^\prime, \vy)\cdot\frac{q^\gets_t(\vy)}{q^\gets_t(\vy^\prime)}$, Eq.~\eqref{eq:rev_func}\\
$\tilde{R}_t(\vy,\vy')$ & Estimated reverse transition rate using the learned density ratio, $\tilde{R}_{t}(\vy,\vy^\prime) = R^\to(\vy^\prime, \vy)\cdot \tilde{v}_{t,\vy^\prime}(\vy)$, Eq.~\eqref{eq:score_estimation}\\
$\hat{R}_t(\cdot,\cdot)$ & Truncated version of $\tilde{R}_t(\cdot,\cdot)$ 
with threshold $\beta_t$, Eq.~\eqref{def:prac_infi_oper_1}  \\
$R^\gets_t(\vy),\ \tilde{R}_t(\vy),\ \hat{R}_t(\vy)$ & Total reverse transition rate out of state $\vy$ for each rate type, defined as $R(\vy)\coloneqq\sum_{\vy' \ne \vy} R(\vy', \vy)$ with $R \in \{R^\gets_t,\ \tilde{R}_t,\ \hat{R}_t\}$ \\
$\beta_t$ & Upper bound on $R^\gets_t(\vy)$, $ \beta_t=  2d\log_2K \max\{1, (T-t)^{-1}\}$, Eq.~\eqref{def:beta_t_} \\
\midrule
$v_{t,\vy'}(\vy)$ & Density ratio $q^\gets_t(\vy) / q^\gets_t(\vy')$ \\
$\tilde{v}_{t,\vy'}(\vy)$ & Learned approximation to $v_{t,\vy'}(\vy) = q^\gets_t(\vy)/q^\gets_t(\vy')$ \\
$L_{\text{SE}}(\hat{v})$ & Score entropy loss used to train $\tilde{v}$, Eq.~\eqref{eq:score_estimation} \\
\midrule
$\ve_i$ & One-hot vector with a $1$ at position $i$ and $0$ elsewhere \\
$l$ & Width of each quantization cell \\
$K = 2L/l$ & Number of quantization bins per dimension \\
\bottomrule
\end{tabular}
\end{table}
\vspace{-5pt}

\section{Technical Lemmas}
\begin{lemma}[Theorem 4.10 of~\cite{boucheron2003concentration}]
    \label{lem:thm4_10_boucheron2003concentration}
    Let $\Phi(x) = x\ln x$ for $x>0$ and $\Phi(0) = 0$. 
    Let $\rvx_1, \rvx_2, \ldots, \rvx_n$ be independent random variables taking values in a countable set $\gX$ and let $f\colon \gX \rightarrow [0, \infty)$.
    We have
    \begin{equation*}
        \begin{aligned}
            &\E_{\rvx_1, \rvx_2, \ldots, \rvx_n}\left[\Phi\left(f(\rvx_1, \rvx_2, \ldots,\rvx_n)\right)\right] - \Phi\left(\E_{\rvx_1,\rvx_2,\ldots, \rvx_n}\left[f\left(\rvx_1,\rvx_2,\ldots,\rvx_n\right)\right]\right)\\
            & \le \sum_{i=1}^n \E_{\rvx_1,\rvx_{i-1},\rvx_{i+1},\ldots,\rvx_n}\left[\E_{\rvx_i}\left[\Phi\left(f(\rvx_1,\rvx_2,\ldots,\rvx_n)\right)\right] - \Phi\left(\E_{\rvx_i}\left[f(\rvx_1,\rvx_2,\ldots,\rvx_n)\right]\right)\right].
        \end{aligned}
    \end{equation*}
\end{lemma}

\begin{lemma}[Chain rule of TV]
    \label{lem:tv_chain_rule}
    Consider four random variables, $\rvx, \rvz, \tilde{\rvx}, \tilde{\rvz}$, whose underlying distributions are denoted as $p_x, p_z, q_x, q_z$.
    Suppose $p_{x,z}$ and $q_{x,z}$ denotes the densities of joint distributions of $(\rvx,\rvz)$ and $(\tilde{\rvx},\tilde{\rvz})$, which we write in terms of the conditionals and marginals as
    \begin{equation*}
        \begin{aligned}
        &p_{x,z}(\vx,\vz) = p_{x|z}(\vx|\vz)\cdot p_z(\vz)=p_{z|x}(\vz|\vx)\cdot p_{x}(\vx)\\
        &q_{x,z}(\vx,\vz)=q_{x|z}(\vx|\vz)\cdot q_z(\vz) = q_{z|x}(\vz|\vx)\cdot q_x(\vx).
        \end{aligned}
    \end{equation*}
    then we have
    \begin{equation*}
        \begin{aligned}
            \TVD{p_{x,z}}{q_{x,z}} \le  \min & \left\{ \TVD{p_z}{q_z} + \E_{\rvz\sim p_z}\left[\TVD{p_{x|z}(\cdot|\rvz)}{q_{x|z}(\cdot|\rvz)}\right],\right.\\
            &\quad  \left.\TVD{p_x}{q_x}+\E_{\rvx \sim p_x}\left[\TVD{p_{z|x}(\cdot|\rvx)}{q_{z|x}(\cdot|\rvx)}\right]\right\}.
        \end{aligned}
    \end{equation*}
    Besides, we have
    \begin{equation*}
        \TVD{p_x}{q_x}\le \TVD{p_{x,z}}{q_{x,z}}.
    \end{equation*}
\end{lemma}

\begin{lemma}[Backward Kolmogorov equation]
    \label{lem:bkw_kolmo}
    Suppose the infinitesimal operator of a Markov semigroup is $\gL$, If we denote the transition density from $\rvy_s= \vy$ to $\rvy_t = \vy^\prime$ as $p_{t|s}(\vy^\prime|\vy)$, then it solves the backward Kolmogorov equation
    \begin{equation*}
        -\frac{\partial p_{t|s}(\vy^\prime|\vy)}{\partial s} = \gL\left[p_{t|s}(\vy^\prime|\cdot)\right](\vy),\quad p_{s|s}(\vy^\prime|\vy) = \delta(\vy^\prime - \vy).
    \end{equation*}
\end{lemma}
\begin{lemma}[Lemma 11 in~\cite{vempala2019rapid}]
    \label{lem:lem11_vempala2019rapid}
    Suppose the density function satisfies $p\propto \exp(-f)$ where $f$ is $H$-smooth, i.e.,~\ref{a2}. 
    Then, it has
    \begin{equation*}
        \E_{\rvx\sim p}\left[\left\|\grad f(\rvx)\right\|^2\right]\le Hd.
    \end{equation*}
\end{lemma}

\section{Forward and Reverse Processes of Discrete Diffusion Models}
\label{sec:app_DisFwdRev}
In order to simplify the notation in this section, we introduce some new notations as supplementary to Section~\ref{sec:pre}.
Since we consider the discrete diffusion on $\gY$, we defined the inner product on this discrete space for two functions as
\begin{equation*}
    \left<f,g\right>_{\gY}\coloneqq \sum_{\vy\in \gY} f(\vy)\cdot g(\vy).
\end{equation*}
Besides, the delta on $\gY$ is defined as
\begin{equation*}
    \delta_{\vy}(\vy^\prime) = \left\{
        \begin{aligned}
            & 1 && \vy^\prime = \vy\\
            & 0 && \text{otherwise}
        \end{aligned}
    \right. .
\end{equation*}

\subsection{The Forward Process of Discrete Diffusion Models}
\label{sec:a1_fwd_and_infi}
In this section, we refine the introduction about the forward process of discrete diffusion in Section~\ref{sec:pre} with the same notations.
In general, the time-homogeneous CTMC 
can be described by a Markov semigroup $\gQ^\to_t$ defined as:
\begin{equation}
    \label{eq:app_semigroup_ope}
    \gQ^\to_{t}[f](\vy) = \E\left[f(\rvy_t)|\rvy_0 =\vy\right] = \left<f, q^\to_{t|0}(\cdot|\vy)\right>_{\gY}
\end{equation}
where the function $f\colon \gY\rightarrow \R$.
Due to the definition, the infinitesimal operator $\gL^\to$ of the time homogeneous $\gQ^\to_t$ is denoted as
\begin{equation}
    \label{def:infinitesimal_oper}
    \gL^\to[f](\vy) = \lim_{t\rightarrow 0} \left[\frac{\gQ^\to_t[f] - f}{t}\right](\vy) = \left<f, \partial_t q^\to_{t|0}(\cdot|\vy)\Big|_{t=0}\right>_{\gY}\coloneqq \left<f, R^\to(\cdot,\vy)\right>_{\gY}
\end{equation}
where 
\begin{equation}
    \label{def:app_mean_of_R}
    R^\to(\vy^\prime, \vy) \coloneqq  \partial_t q^\to_{t|0}(\vy^\prime|\vy)\Big|_{t=0} = \lim_{t\rightarrow 0}\left[\frac{q^\to_{t|0}(\vy^\prime|\vy)-\delta_{\vy}(\vy^\prime)}{t}\right].
\end{equation}
According to the time-homogeneous property, we have
\begin{equation*}
    q^\to_{t+\Delta t | t}(\vy^\prime|\vy) = \delta_{\vy}(\vy^\prime) + \Delta t\cdot R^\to(\vy^\prime, \vy)  + o(\Delta t)
\end{equation*}
for any $t$.
Here, the transition rate function $R^\to$ must satisfy
\begin{equation}
    \label{eq:transition_rate_prop}
    R^\to(\vy,\vy^\prime)\ge 0 \  \text{when}\  \vy^\prime\not=\vy\quad \text{and}\quad R^\to(\vy^\prime,\vy^\prime) = -\sum_{\vy\not=\vy^\prime}R^\to(\vy,\vy^\prime)\le 0
\end{equation}
due to the definition Eq.~\eqref{def:app_mean_of_R}.
Under this setting, we can provide the dynamic of $q_{t|0}$ for any $t$.
Specifically, we have
\begin{equation*}
    \begin{aligned}
        &\partial_t \gQ^\to_t[f](\vy) = \gQ^\to_t\left[\gL f\right](\vy) = \left<\gL^\to f, q^\to_{t|0}(\cdot|\vy)\right>_{\gY} = \sum_{\vy^\prime \in \gY} \gL^\to[f](\vy^\prime)\cdot q^\to_{t|0}(\vy^\prime|\vy)\\
        & = \sum_{y^\prime\in \gY}\left[  \sum_{\tilde{\vy}\in\gY} f(\tilde{\vy})\cdot R^\to(\tilde{\vy}, \vy^\prime) \cdot q_{t|0}(\vy^\prime|\vy)\right] = \sum_{\tilde{\vy}\in\gY}\left[f(\tilde{\vy})\cdot \sum_{\vy^\prime\in \gY}R^\to(\tilde{\vy},\vy^\prime)\cdot q_{t|0}(\vy^\prime|\vy)\right],
    \end{aligned}
\end{equation*}
where the first inequality follows from the semigroup property.
Combining with the fact
\begin{equation*}
    \partial_t\gQ^\to_t[f](\vy) = \left<f, \partial_t q^\to_{t|0}(\cdot|\vy)\right>_{\gY}
\end{equation*}
derived from Eq.~\eqref{eq:app_semigroup_ope},
we have
\begin{equation*}
    \partial_t q^\to_{t|0}(\tilde{\vy}|\vy) = \sum_{\vy^\prime\in \gY}R(\tilde{\vy},\vy^\prime)\cdot q^\to_{t|0}(\vy^\prime|\vy) = \left<R(\tilde{\vy},\cdot), q^\to_{t|0}(\cdot|\vy)\right>_{\gY}.
\end{equation*}
According to the time-homogeneous property, the above equation can be easily extended to 
\begin{equation}
    \label{eq:condi_fwd_dis}
    \partial_t q^\to_{t|s}(\tilde{\vy}|\vy) = \sum_{\vy^\prime\in \gY}R(\tilde{\vy},\vy^\prime)\cdot q^\to_{t|s}(\vy^\prime|\vy) = \left<R(\tilde{\vy},\cdot), q^\to_{t|s}(\cdot|\vy)\right>_{\gY}.
\end{equation}
Combining with Bayes' Theorem, the transition of the marginal distribution is
\begin{equation}
    \label{eq:app_fwd_func}
    \frac{\der q^\to_t}{\der t}(\vy) = \left<R(\vy,\cdot), q^\to_t\right>_{\gY}.
\end{equation}
\paragraph{Matrix Presentation.} Suppose the support set $\gY$ of $q^\to_t$ be written as $\gY =\{\vy_0, \vy_1,\ldots, \vy_{|\gX|}\}$, we may consider the marginal distribution $q^\to_s$ to be a vector, i.e.,
\begin{equation*}
    \vq^\to_t = \left[q_t(\vy_0), q_t(\vy_1),\ldots, q_t(\vy_{|\gY|-1})\right],
\end{equation*}
conditional transition probability function $q^\to_{t|s}$ to be a matrix, i.e.,
\begin{equation*}
    \mQ^\to_{t|s} = \left[
        \begin{matrix}
            q^\to_{t|s}(\vy_0|\vy_0) & q^\to_{t|s}(\vy_0|\vy_1) & \ldots & q^\to_{t|s}(\vy_0|\vy_{|\gY|-1})\\
            q^\to_{t|s}(\vy_1|\vy_0) & q^\to_{t|s}(\vy_1|\vy_1) & \ldots & q^\to_{t|s}(\vy_1|\vy_{|\gY|-1})\\
            \ldots & \ldots & \ldots & \ldots\\
            q^\to_{t|s}(\vy_{|\gY|-1}|\vy_0) & q^\to_{t|s}(\vy_{|\gY|-1}|\vy_1) & \ldots & q^\to_{t|s}(\vy_{|\gY|-1}|\vy_{|\gY|-1})\\
        \end{matrix}
    \right].
\end{equation*}
Similarly, the function $R$ can also be presented as
\begin{equation}
    \label{def:trans_ker_to_matrix}
    \mR^\to = \left[
        \begin{matrix}
            R^\to(\vy_0,\vy_0) & R^\to(\vy_0, \vy_1) & \ldots & R^\to(\vy_0, \vy_{|\gY|-1})\\
            R^\to(\vy_1,\vy_0) & R^\to(\vy_1, \vy_1) & \ldots & R^\to(\vy_1, \vy_{|\gY|-1})\\
            \ldots & \ldots & \ldots & \ldots\\
            R^\to(\vy_{|\gY|-1},\vy_0) & R^\to(\vy_{|\gY|-1}, \vy_1) & \ldots & R^\to(\vy_{|\gY|-1}, \vy_{|\gY|-1})\\
        \end{matrix}
    \right].
\end{equation}
Under this condition, Eq.~\eqref{eq:app_fwd_func} can be written as 
\begin{equation}
    \label{eq:fwd_vec}
    \der \vq^\to_t/\der t = \mR^\to\cdot \vq^\to_t
\end{equation}
matching the usual presentation shown in~\cite{chen2024convergence,zhang2024convergence}.
Besides, Eq.~\eqref{eq:transition_rate_prop} shown in Section~\ref{sec:pre} can also be presented as $\vone\cdot \mR =\vzero$.


The following lemma gives the closed-form expression for the probability transition kernel of the forward process, which also suggests an efficient implementation.

\begin{lemma}[Forward transition kernel]
    \label{lem:fwd_trans_ker}
    Consider the forward CTMC, i.e., $\{\rvy_t\}_{t=0}^T$ with the infinitesimal operator $\R^\to$ given in Eq.~\eqref{eq:fwd_transtion_rate_max}. Then for any two timestamps $s\le t$, the forward transition probability satisfies
    \begin{equation*}
        q^\to_{t|s}(\vy|\vy^\prime) =2^{-d\log_2K}\cdot \prod_{i=0}^{d\log_2K-1}\left[1+(-1)^{|\vy_i - \vy_i^\prime|}\cdot e^{-2(t-s)}\right].
    \end{equation*}
\end{lemma}

\begin{remark}
The transition probability in Lemma~\ref{lem:fwd_trans_ker} factorizes across coordinates. This means that the forward transition can be implemented as $d \log_2 K$ independent bit-wise updates. Specifically, for each coordinate $i$, flip $\vy_i^\prime$ with probability $\frac{1 - e^{-2(t - s)}}{2}$ to obtain $\vy_i$.
\end{remark}
\begin{proof}
    Combining Eq.~\eqref{def:trans_ker_to_matrix} and Eq.~\eqref{eq:fwd_vec}, the dynamic of marginal distribution $q^\to_t$ can be written as a matrix-vector product, i.e.,
    \begin{equation*}
        \der \vq^\to_t/\der t = \mR^\to\cdot \vq^\to_t
    \end{equation*}
    where
    \begin{equation*}
        \mR^\to = \left[
        \begin{matrix}
            R^\to(\vy_0,\vy_0) & R^\to(\vy_0, \vy_1) & \ldots & R^\to(\vy_0, \vy_{|\gY|-1})\\
            R^\to(\vy_1,\vy_0) & R^\to(\vy_1, \vy_1) & \ldots & R^\to(\vy_1, \vy_{|\gY|-1})\\
            \ldots & \ldots & \ldots & \ldots\\
            R^\to(\vy_{|\gY|-1},\vy_0) & R^\to(\vy_{|\gY|-1}, \vy_1) & \ldots & R^\to(\vy_{|\gY|-1}, \vy_{|\gY|-1})\\
        \end{matrix}
    \right].
    \end{equation*}
    Here, $\mR^\to$ can be decomposed into the sum 
    \begin{equation*}
        \mR^\to = \sum_{i=0}^{d\log_2K-1} \mR^\to_i,
    \end{equation*}
    we first note that the state space is $\{0,1\}^{d\log_2K}$, where each coordinate can flip independently. 
    Hence, each coordinate contributes its own ``flip'' component to the overall generator $\mR^\to$. 
    Concretely, let us label the coordinates $0, \dots, d\log_2K-1$, and consider the generator corresponding to a single coordinate $i$. 
    Such a generator only acts nontrivially on the $i$th coordinate, which can flip from $0$ to $1$ or $1$ to $0$, while all other coordinates remain unchanged.

    Each ``flip'' for coordinate $i$ can be represented by a $2\times 2$ generator matrix (reflecting the two possible states, $0$ or $1$). 
    Moreover, since the flipping of different coordinates occurs independently, we adopt the tensor-product (or Kronecker-product) structure, placing the $2\times 2$ flip matrix in the $i$th position and $2\times 2$ identity matrices in all other positions. 
    Hence, each $\mR^\to_i$ is of the form
    \begin{equation*}
        \mR^\to_i =  \mI \otimes\, \cdots \,\otimes\, \mA \,\otimes\, \cdots \,\otimes\, \mI,
    \end{equation*}
    where 
    \begin{equation*}
        \mA \coloneqq \left[
        \begin{matrix}
            -1 & 1\\
            1 & -1
        \end{matrix}
        \right]
    \end{equation*}
    is a generator of the flip in the $i$th coordinate, and $\mI$ is the $2 \times 2$ identity in all coordinates. 
    By the Kolmogorov forward equation, we have
    \begin{equation*}
        \mQ^\to_{t|s} = \exp\left((t-s)\mR^\to\right) = \exp\left((t-s) \mA\right)^{\otimes d} = \left[
            \begin{aligned}
                & \frac{1+e^{-2(t-s)}}{2} && \frac{1-e^{-2(t-s)}}{2} \\
                & \frac{1-e^{-2(t-s)}}{2} && \frac{1+e^{-2(t-s)}}{2} \\
            \end{aligned}
        \right]^{\otimes d} ,
    \end{equation*}
    which implies 
    \begin{equation*}
        q^\to_{t|s}(\vy|\vy^\prime) =2^{-d\log_2K}\cdot \prod_{i=0}^{d\log_2K-1}\left[1+(-1)^{|\vy_i - \vy_i^\prime|}\cdot e^{-2(t-s)}\right]\quad \text{and}\quad \vy,\vy^\prime\in \gY.
    \end{equation*}
    Hence, the proof is completed.
\end{proof}

Figure~\ref{fig:graph} visualizes the evolution of transition probabilities under different forward processes. The tridiagonal CTMC (second row) can be viewed as a discrete analogue of the normalized Gaussian transition (first row), where the domain $[0,1]$ is quantized into 8 bins.
The tridiagonal structure results in slow mixing, as transitions are restricted to immediate neighbors. At small time steps (e.g., $t = 0.01$, first column), the transition kernel satisfies $\mQ^\to_{t+\Delta t|t} \approx \mI + \Delta t \cdot \mR^\to$, so the sparsity of the transition kernel closely reflects that of the rate matrix $\mR^\to$. For efficient simulation of the reverse process, defined by $R_t^\gets(\vy, \vy^\prime) \coloneqq   R^\to(\vy^\prime, \vy)\cdot\frac{q^\gets_t(\vy)}{q^\gets_t(\vy^\prime)}$ as Eq.~\eqref{eq:rev_func}, it is essential that $\mR^\to$ remains sparse.
While the dense forward process (third row) mixes rapidly, it incurs high computational cost per step when simulating the reverse process. In contrast, the hypercube structure (fourth row) achieves a favorable balance: it enables efficient long-range transitions for fast mixing while preserving an $\mathcal{O}(\log |\gY|)$ sparse structure for efficient implementation.

\begin{figure}[H]
    \centering
    \includegraphics[width=0.99\linewidth]{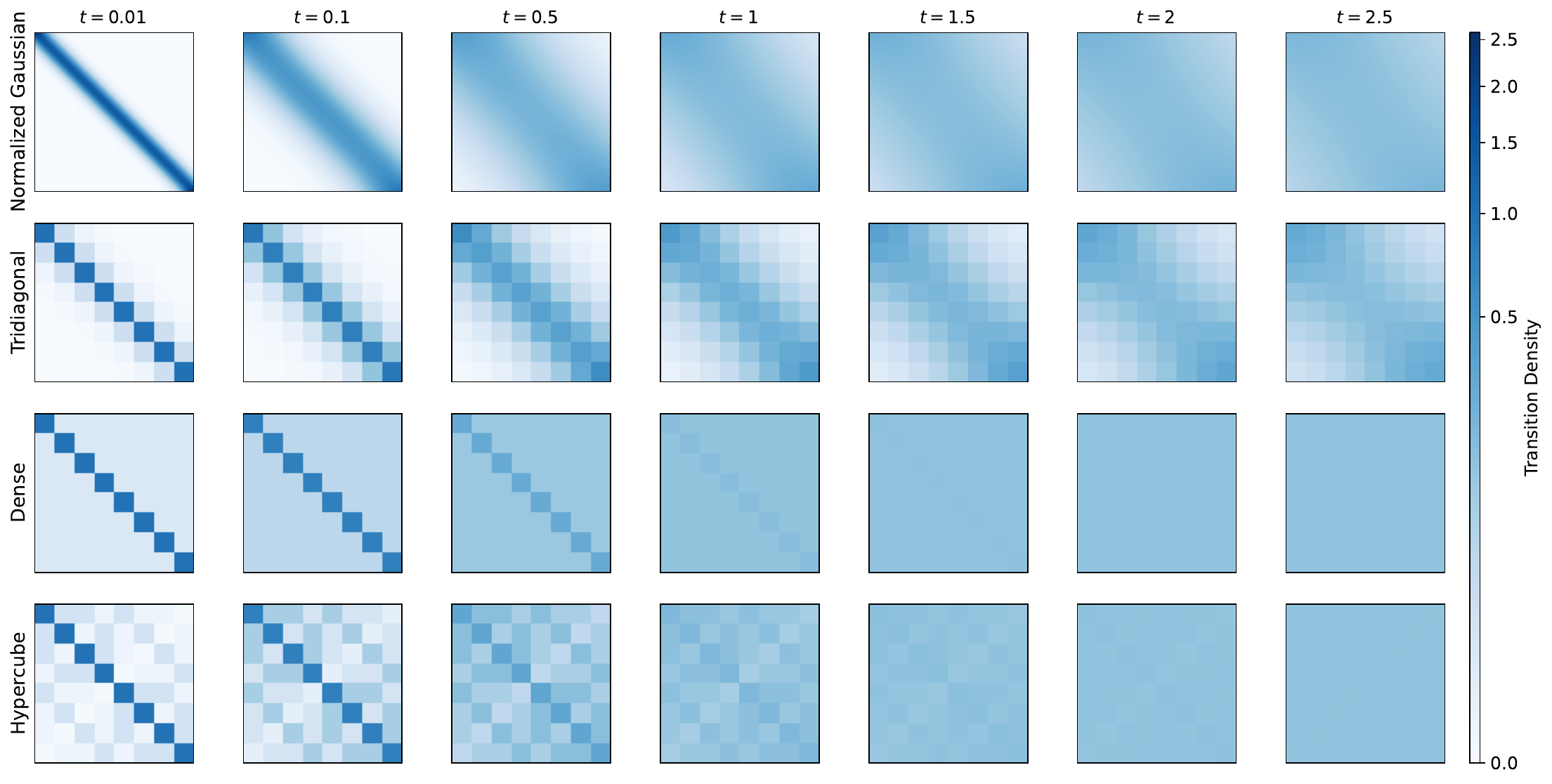}
    \caption{
    Heatmaps of the probability transition at different time steps $t$ for four diffusion processes: a continuous normalized Gaussian kernel on $[0,1]$ (top row), and discrete CTMCs over $|\gY| = 8$ states based on tridiagonal, dense, and hypercube transition rate matrices (bottom three rows).
    }
    \label{fig:graph}
\end{figure}

\subsection{Proof of Eq.~\eqref{eq:rev_func}}
\label{sec:a2_reverse_prop}

\begin{proof}
For any $t\in[0,T]$, the marginal, joint, and conditional distribution w.r.t. $\{\rvy_t^\gets\}$ are denoted as
\begin{equation*}
    \rvy^\gets_t\sim q^\gets_t,\quad (\rvy^\gets_t, \rvy^\gets_{t^\prime}) \sim q^\gets_{t,t^\prime}, \quad \text{and}\quad q^\gets_{t^\prime|t} = q_{t^\prime, t}/q_t,
\end{equation*}
which have $q^\gets_t = q^\to_{T-t}$.
Then, we start to check the dynamic of $q^\gets_{t|s}$, i.e.,
\begin{equation}
    \label{ineq:discrete_rev_0}
    \begin{aligned}
        & \partial_t q_{t|s}^\gets(\vy^\prime|\vy) = -1\cdot \partial_{T-t} q^\to_{T-t|T-s}(\vy^\prime|\vy) = -1\cdot \partial_{T-t}\left[\frac{q^\to_{T-s|T-t}(\vy|\vy^\prime)\cdot q^\to_{T-t}(\vy^\prime)}{q^\to_{T-s}(\vy)}\right]\\
        & = - \underbrace{\partial_{T-t} q^\to_{T-s|T-t}(\vy|\vy^\prime)\cdot \frac{q^\to_{T-t}(\vy^\prime)}{q^\to_{T-s}(\vy)}}_{\text{Term 1}} - \underbrace{\frac{q^\to_{T-s|T-t}(\vy|\vy^\prime)}{q^\to_{T-s}(\vy)}\cdot \partial_{T-t} q^\to_{T-t}(\vy^\prime)}_{\text{Term 2}}.
    \end{aligned}
\end{equation}
For $\text{Term 1}$ of Eq.~\eqref{ineq:discrete_rev_0}, we have
\begin{equation*}
    \begin{aligned}
        \text{Term 1} & = -\sum_{\tilde{\vy}\in \gY} R^\to(\tilde{\vy}, \vy^\prime)\cdot q^\to_{T-s|T-t}(\vy|\tilde{\vy})\cdot \frac{q^\to_{T-t}(\tilde{\vy})}{q^\to_{T-s}(\vy)}\cdot \frac{q^\to_{T-t}(\vy^\prime)}{q^\to_{T-t}(\tilde{\vy})}\\
        & = -\sum_{\tilde{\vy}\in \gY} R^\to(\tilde{\vy}, \vy^\prime)\cdot\frac{q^\to_{T-t}(\vy^\prime)}{q^\to_{T-t}(\tilde{\vy})}\cdot q^\to_{T-t|T-s}(\tilde{\vy}|\vy),
    \end{aligned}
\end{equation*}
where the first equation follows from the Kolmogorov backward theorem (Lemma~\ref{lem:bkw_kolmo}) and Eq.~\eqref{def:infinitesimal_oper}:
\begin{equation*}
    \partial_{T-t} q^\to_{T-s|T-t}(\vy|\vy^\prime) = -\gL^\to[q^\to_{T-s|T-t}(\vy| \cdot)](\vy^\prime) = - \left<q^\to_{T-s|T-t}(\vy|\cdot), R^\to(\cdot,\vy^\prime)\right>_{\gY}.
\end{equation*}
For Term 2 of Eq.~\eqref{ineq:discrete_rev_0}, we have
\begin{equation*}
    \begin{aligned}
        \text{Term 2} & =  \frac{q^\to_{T-s|T-t}(\vy|\vy^\prime)}{q^\to_{T-s}(\vy)}\cdot \sum_{\tilde{\vy}\in \gY} R^\to(\vy^\prime, \tilde{\vy})\cdot q^\to_{T-t}(\tilde{\vy})\\
        & = \frac{q^\to_{T-s|T-t}(\vy|\vy^\prime)\cdot q^\to_{T-t}(\vy^\prime)}{q^\to_{T-s}(\vy)}\cdot \sum_{\tilde{\vy}\in\gY} R^\to(\vy^\prime, \tilde{\vy})\cdot \frac{q^\to_{T-t}(\tilde{\vy})}{q^\to_{T-t}(\vy^\prime)} = 0,
    \end{aligned}
\end{equation*}
where the first equation follows from Eq.~\eqref{eq:fwd_func} and the last equation follows from the fact
\begin{equation*}
    \begin{aligned}
        & \sum_{\tilde{\vy}\in\gY} R^\to(\vy^\prime, \tilde{\vy})\cdot \frac{q^\to_{T-t}(\tilde{\vy})}{q^\to_{T-t}(\vy^\prime)} = \sum_{\tilde{\vy}\in \gY} \lim_{t\rightarrow 0}\left[\frac{q^\to_{t|0}(\vy^\prime|\tilde{\vy})-\delta_{\tilde{\vy}}(\vy^\prime)}{t}\right] \cdot \frac{q^\to_{T-t}(\tilde{\vy})}{q^\to_{T-t}(\vy^\prime)} \\
        & =  \sum_{\tilde{\vy}\in \gY} 
 \lim_{t^\prime\rightarrow T-t} \left[\frac{q^\to_{t^\prime|T-t}(\vy^\prime|\tilde{\vy})-\delta_{\tilde{\vy}}(\vy^\prime)}{t^\prime - (T-t)}\right] \cdot \lim_{t^\prime \rightarrow T-t} \frac{q^\to_{T-t}(\tilde{\vy})}{q^\to_{t^\prime}(\vy^\prime)} = \sum_{\tilde{\vy}\in \gY} \lim_{t^\prime\rightarrow T-t} \left[\frac{q^\to_{T-t|t^\prime}(\tilde{\vy}|\vy^\prime) - \delta_{\vy^\prime}(\tilde{\vy})}{t^\prime - (T-t)}\right] = 0.
    \end{aligned}
\end{equation*}
Under this condition, by setting 
\begin{equation*}
    \begin{aligned}
        R_t^\gets(\vy^\prime, \tilde{\vy}) \coloneqq   R(\tilde{\vy}, \vy^\prime)\cdot\frac{q^\gets_t(\vy^\prime)}{q^\gets_t(\tilde{\vy})},
    \end{aligned}
\end{equation*}
then Eq.~\eqref{ineq:discrete_rev_0} can be summarized as
\begin{equation}
    \label{eq:condi_rev_func}
    \begin{aligned}
        \partial_t q_{t|s}^\gets(\vy^\prime|\vy) = \left<R^\gets_t(\vy^\prime, \cdot), q^\gets_{t|s}(\cdot|\vy)\right>_{\gY} = \sum_{\tilde{\vy}\in \gY} R^\gets_{t}(\vy^\prime, \tilde{\vy})\cdot q_{t|s}^\gets(\tilde{\vy}|\vy).
    \end{aligned}
\end{equation}
Combining with the Bayes' Theorem, we have
\begin{equation}
    \label{eq:app_rev_func}
    \frac{\der q^\gets_t}{\der t}(\vy) = \left<R^\gets_t(\vy,\cdot), q^\gets_t\right>_{\gY}.
\end{equation}
Hence, Eq.~\eqref{eq:rev_func} establishes.
\end{proof}

\subsection{Proof of Eq.~\eqref{def:mean_of_R_gets}}
\label{sec:app_prof_mean_of_r_gets}
\begin{proof}[Adapted from Proposition 1 of~\cite{campbell2022continuous}]
    The RHS of Eq.~\eqref{def:mean_of_R_gets} satisfies
    \begin{equation*}
        \lim_{\Delta t\rightarrow 0}\left[\frac{q^\gets_{t+\Delta t|t}(\vy|\vy^\prime)-\delta_{\vy^\prime}(\vy)}{\Delta t}\right] = \lim_{s\rightarrow t} \partial_t q^\gets_{t|s}(\vy|\vy^\prime).
    \end{equation*}
    Besides, we have
    \begin{equation*}
        \begin{aligned}
            &\lim_{s\rightarrow t} \partial_t q^\gets_{t|s}(\vy|\vy^\prime) = \lim_{s\rightarrow t} \partial_t \left[q^\to_{T-s|T-t}(\vy^\prime|\vy)\cdot \frac{q^\to_{T-t}(\vy)}{q^\to_{T-s}(\vy^\prime)}\right]\\
            & = \lim_{s\rightarrow t}\left[\partial_t(q^\to_{T-s|T-t}(\vy^\prime|\vy))\cdot \frac{q^\to_{T-t}(\vy)}{q^\to_{T-s}(\vy^\prime)} + q^\to_{T-s|T-t}(\vy^\prime|\vy)\cdot \frac{\partial_t q^\to_{T-t}(\vy)}{q^\to_{T-s}(\vy^\prime)}\right].
        \end{aligned}
    \end{equation*}
    When $\vy\not=\vy^\prime$, we have
    \begin{equation*}
        \lim_{s\rightarrow t}q^\to_{T-s|T-t}(\vy^\prime|\vy) = 0,
    \end{equation*}
    which implies
    \begin{equation*}
        \lim_{s\rightarrow t} \partial_t q^\gets_{t|s}(\vy|\vy^\prime) =\lim_{s\rightarrow t}\partial_t(q^\to_{T-s|T-t}(\vy^\prime|\vy))\cdot \frac{q^\to_{T-t}(\vy)}{q^\to_{T-s}(\vy^\prime)} = R^\to(\vy^\prime, \vy)\cdot \frac{q^\to_{T-t}(\vy)}{q^\to_{T-t}(\vy^\prime)}.
    \end{equation*}
    The last equation follows from the Kolmogorov backward theorem, i.e., Lemma~\ref{lem:bkw_kolmo} and Eq.~\eqref{def:infinitesimal_oper}
    \begin{equation*}
        \partial_{T-t} q^\to_{T-s|T-t}(\vy^\prime|\vy) = -\gL^\to[q^\to_{T-s|T-t}(\vy^\prime| \cdot)](\vy) = - \left<q^\to_{T-s|T-t}(\vy^\prime|\cdot), R^\to(\cdot,\vy)\right>_{\gY} = R^\to(\vy^\prime,\vy).
    \end{equation*} 
    Combining with Eq.~\eqref{eq:rev_func}, we have
    \begin{equation}
        \label{eq:R_gets_temp1}
        \lim_{\Delta t\rightarrow 0}\left[\frac{q^\gets_{t+\Delta t|t}(\vy|\vy^\prime)-\delta_{\vy^\prime}(\vy)}{\Delta t}\right] = \lim_{s\rightarrow t} \partial_t q^\gets_{t|s}(\vy|\vy^\prime)= R^\to(\vy^\prime, \vy)\cdot \frac{q^\to_{T-t}(\vy)}{q^\to_{T-t}(\vy^\prime)} = R^\gets_{t}(\vy,\vy^\prime)
    \end{equation}
    when $\vy^\prime\not=\vy$.
    Besides, we have
    \begin{equation*}
        \begin{aligned}
            &\sum_{\vy\in\gY} R^\gets_t(\vy,\vy^\prime) = \sum_{\vy\in\gY}R^\to(\vy^\prime,\vy)\cdot \frac{q_{T-t}^\to(\vy)}{q_{T-t}^\to(\vy^\prime)}\\
            & = \sum_{\vy\in\gY} \lim_{\Delta t\rightarrow 0}\left[\frac{q^\to_{T-t+\Delta t|T-t}(\vy^\prime|\vy)-\delta_{\vy}(\vy^\prime)}{\Delta t}\right]\cdot \frac{q_{T-t}^\to(\vy)}{q_{T-t}^\to(\vy^\prime)} = \sum_{\vy\in\gY} \lim_{\Delta t\rightarrow 0}\left[\frac{q^\to_{T-t+\Delta t|T-t}(\vy|\vy^\prime)-\delta_{\vy^\prime}(\vy)}{\Delta t}\right]=0,
        \end{aligned}
    \end{equation*}
    which means 
    \begin{equation*}
        R_t^\gets(\vy^\prime,\vy^\prime) = -\sum_{\vy\not=\vy^\prime} R^\gets_t(\vy,\vy^\prime) = \lim_{\Delta t\rightarrow 0} - \left[\frac{1-\sum_{\vy\not=\vy^\prime} q^\gets_{t+\Delta t|t}(\vy|\vy^\prime)}{\Delta t}\right],
    \end{equation*}
    where the last inequality follows from Eq.~\eqref{eq:R_gets_temp1}.
    Hence, the proof is completed.
\end{proof}

\section{Proof of Lemma~\ref{lem:discrete_quantization_gap}}
\label{sec:discrete_quantization_gap}

\begin{lemma}
    \label{lem:histog_dis_comp}
    Suppose the data distribution $p_*$ is $\sigma$ sub-Gaussian, by choosing
    $L\ge \sigma\cdot \sqrt{2\ln(2d/\epsilon)}$, the TV distance between $p_*$ and $\tilde{p}_*$ defined in Eq.~\eqref{def:bounded_approx} will be smaller than $\epsilon$, i.e., $\TVD{p_*}{\tilde{p}_*}\le \epsilon$.
\end{lemma}
\begin{proof}
    When $p_*$ satisfies $\sigma$ sub-Gaussian properties, i.e.,
    \begin{equation*}
        \E_{\rvx\sim p_*}\left[\exp\left(l\left<\rvx, \vu\right>\right)\right]\le \exp\left(\frac{\sigma^2l^2 \cdot \|\vu\|^2}{2}\right).
    \end{equation*}
    By choosing $\vu = \ve_i$, we can easily found that each dimension of $\rvx$ will be $\sigma$ sub-Gaussian, i.e.,
    \begin{equation*}
        \E_{\rvx_i \sim p_{*,i}}\left[\exp\left(l\rvx_i \vu_i\right)\right]\le \exp\left(\frac{\sigma^2l^2 \cdot \|\vu_i\|^2}{2}\right).
    \end{equation*}
    According to the sub-Gaussian properties for each coordinate, we have
    \begin{equation*}
        \Pr_i\left[\left|\rvx_i\right|\ge l\right]\le 2\exp\left(-\frac{l^2}{2\sigma^2}\right).
    \end{equation*}
    With the union bound, we have
    \begin{equation*}
        \Pr_*\left[\max_{1\le i\le d}\ \left|\rvx_i\right| >L\right]\le\sum_{i=1}^d \Pr_*\left[|\rvx_i| > L\right] \le 2d\cdot \exp\left(-\frac{L^2}{2\sigma^2}\right).
    \end{equation*}
    Under this condition, by supposing 
    \begin{equation}
        \label{ineq:prob_beyond_cube_}
        2d\cdot \exp\left(-\frac{L^2}{2\sigma^2}\right) \le \epsilon \quad \Leftrightarrow\quad L\ge \sigma\cdot \sqrt{2\ln \frac{2d}{\epsilon}},
    \end{equation}
    we have $\Pr_*\left[\max_{1\le i\le d}\left|\rvx_i\right|\ge L\right]\le\epsilon$.
    Under this condition, the total variation distance between $\tilde{p}_*$ and $p_*$ can be upper bounded by 
    \begin{equation}
        \label{ineq:tv_gap_tilde_ori}
        \begin{aligned}
            &\TVD{p_*}{\tilde{p}_*} =\frac{1}{2} \int_{\R^d} \left|p_*(\vx) - \tilde{p}_*(\vx)\right| \der \vx \\
            & = \frac{1}{2}\int_{\vx\in\Cb{L}} \left(\tilde{p}_*(\vx) - p_*(\vx)\right)\der \vx + \frac{1}{2}\int_{\vx\not\in \Cb{L}} p_*(\vx) \der\vx\\
            & = \frac{1}{2}\left[1-\int_{\vx\in\Cb{L}} p_*(\vx) \der \vx\right]+ \frac{1}{2}\int_{\vx\not\in \Cb{L}} p_*(\vx) \der\vx\\
            & = \int_{\vx\not\in \Cb{L}} p_*(\vx)\der \vx \le \epsilon
        \end{aligned}
    \end{equation}
    where the last inequality follows from Eq.~\eqref{ineq:prob_beyond_cube_}. Hence, the proof is completed.
\end{proof}

\begin{lemma}
    \label{lem:histog_dis_quan}
    Suppose the distribution $\tilde{p}_*$ defined in Eq.~\eqref{def:bounded_approx} satisfies $H$--smoothness, by choosing
    \begin{equation*}
        l \le (2HL+\|\grad f_*(\vzero)\|)^{-1}\cdot d^{-1/2}\epsilon,
    \end{equation*}
    the TV distance satisfies $\TVD{\tilde{p}_*}{\overline{p}_*}\le 2\epsilon$ where  $\overline{p}_*$ is defined in Eq.~\eqref{def:quantize_cell}.
\end{lemma}
\begin{proof}
By Lagrange's mean value theorem, for each cell $\Ce{i_0, i_1,\ldots,i_{d-1}}$, there exists a point $\Bar{\vx}_{{i_0, i_1,\ldots,i_{d-1}}}\in \Ce{i_0, i_1,\ldots,i_{d-1}}$ such that 
\[
\tilde{p}_*(\Bar{\vx}_{{i_0, i_1,\ldots,i_{d-1}}})=\frac{\int_{\vu \in \Ce{i_0, i_1,\ldots,i_{d-1}}}\tilde{p}_*(\vu)\der \vu}{l^d}.
\]
Therefor, the piecewise constant density $\overline{p}_*$ satisfies $\overline{p}_*(\vx)=\tilde{p}_*(\Bar{\vx}_{{i_0, i_1,\ldots,i_{d-1}}})$, for any $\vx \in \Ce{i_0, i_1,\ldots,i_{d-1}}$.

We now aim to bound the difference $\left|\tilde{p}_*(\vu)-\tilde{p}_*(\vx)\right|$ for any $\vu, \vx \in \Ce{i_0, i_1,\ldots,i_{d-1}}$, using $H$--smoothness. Later, we will choose $\vu = \Bar{\vx}_{i_0, i_1,\ldots,i_{d-1}}$ to bound the total variation distance between $\tilde{p}_*$ and $\overline{p}_*$.

According to the construction of $\tilde{p}_*$, i.e., Eq.~\eqref{def:bounded_approx}, we have
\begin{equation}
    \label{eq:density_ratio_mid}
    \frac{\tilde{p}_*(\vu)}{\tilde{p}_*(\vx)} = \frac{p_*(\vu)}{p_*(\vx)} = \exp\left(f_*(\vx)-f_*(\vu)\right).
\end{equation}
With $H$--smoothness, i.e., $\left\|\grad^2 f_*\right\|\le H$, we have
\begin{equation}
    \label{ineq:density_ratio_mid_upb}
    \begin{aligned}
        & f_*(\vx) - f_*(\vu) \le \grad f_*(\vu)\cdot \left(\vx - \vu\right) + \frac{H}{2}\cdot \left\|\vu - \vx\right\|^2\\
        & \le \left\|\grad f_*(\vu)\right\|\cdot \left\|\vx- \vu\right\| + \frac{H}{2}\cdot \left\|\vx- \vu\right\|^2.
    \end{aligned}
\end{equation}
Since $\vu, \vx \in \Ce{i_0, i_1,\ldots,i_{d-1}}$, and each cell is an axis-aligned hypercube of side length $l$, we have
\begin{equation*}
    \left\|\vx - \vu\right\|^2 = \sum_{i=1}^d \left\|\vx_i - \vu_i\right\|^2 \le dl^2.
\end{equation*}
Let $G_0\coloneqq \|\grad f_*(\vzero)\|$. Then we have
\begin{equation*}
    \left\|\grad f_*(\vu)\right\|\le \left\|\grad f_*(\vu) - \grad f_*(\vzero)\right\| + G_0\le H\cdot 2L + G_0,
\end{equation*}
where the last inequality follows from $\vu\in \Cb{L}$.
Therefore, by requiring 
\begin{equation*}
    l \le \frac{\epsilon}{\sqrt{d}\cdot (2HL+G_0)},
\end{equation*}
and $\epsilon\leq 8 H L^2$ without loss of generality, we will have $l\leq \sqrt{2\epsilon/(d H)}$, which means
\begin{equation}
    \label{ineq:density_ratio_mid_fin}
   \left\|\grad f_*(\vu)\right\|\cdot \left\|\vx- \vu\right\| + \frac{H}{2}\cdot \left\|\vx- \vu\right\|^2\le (2HL+G_0)\cdot \sqrt{d}l + \frac{H}{2}\cdot dl^2 \le 2\epsilon.
\end{equation}
Plugging Eq.~\eqref{ineq:density_ratio_mid_upb} and Eq.~\eqref{ineq:density_ratio_mid_fin} into Eq.~\eqref{eq:density_ratio_mid}, we have
\begin{equation}
    \label{ineq:tilde_ratio_upper}
    \frac{\tilde{p}_*(\vu)}{\tilde{p}_*(\vx)}\le \exp(2\epsilon)\le (1+4\epsilon).
\end{equation}
With a similar technique, we have
\begin{equation*}
    \begin{aligned}
        -(f_*(\vx) - f_*(\vu)) = f_*(\vu) - f_*(\vx) \le \left\|\grad f_*(\vx)\right\|\cdot \left\|\vx- \vu\right\| + \frac{H}{2}\cdot \left\|\vx- \vu\right\|^2.
    \end{aligned}
\end{equation*}
Under the same setting, it implies
\begin{equation}
    \label{ineq:tilde_ratio_lower}
    \frac{\tilde{p}_*(\vx)}{\tilde{p}_*(\vu)}\le \exp(2\epsilon) \quad \Leftrightarrow\quad  \frac{\tilde{p}_*(\vu)}{\tilde{p}_*(\vx)}\ge \exp(-2\epsilon)\ge 1-2\epsilon.
\end{equation}
Combining Eq.~\eqref{ineq:tilde_ratio_upper} with Eq.~\eqref{ineq:tilde_ratio_lower}, we have
\begin{equation}
    \label{ineq:tilde_p_ratio_lr_bound}
    1-4\epsilon\le \frac{\tilde{p}_*(\vu)}{\tilde{p}_*(\vx)} \le 1+4\epsilon.
\end{equation}

Hence we are able to control the TV distance between $\tilde{p}_*$ and $\overline{p}_*$, i.e.,
\begin{equation}
    \label{ineq:tv_gap_tilde_overline}
    \begin{aligned}
        &\TVD{\overline{p}_*}{\tilde{p}_*}= \frac{1}{2}\int_{\vx\in \Cb{L}} \left|\overline{p}_*(\vx) - \tilde{p}_*(\vx)\right| \der \vx \\
        & = \frac{1}{2} \sum_{i_0, i_1, \dots, i_{d-1}} \int_{\vx \in \Ce{i_0, i_1, \dots, i_{d-1}}} \left|\overline{p}_*(\vx) - \tilde{p}_*(\vx)\right| \der \vx \\
        & = \frac{1}{2} \sum_{i_0, i_1, \dots, i_{d-1}} \int_{\vx \in \Ce{i_0, i_1, \dots, i_{d-1}}} \left|\tilde{p}_*(\Bar{\vx}_{i_0, i_1,\ldots,i_{d-1}}) - \tilde{p}_*(\vx)\right| \der \vx\\
        & = \frac{1}{2} \sum_{i_0, i_1, \dots, i_{d-1}} \int_{\vx \in \Ce{i_0, i_1, \dots, i_{d-1}}} \tilde{p}_*(\vx)\left|\frac{\tilde{p}_*(\Bar{\vx}_{i_0, i_1,\ldots,i_{d-1}})}{\tilde{p}_*(\vx)} - 1\right| \der \vx\\ 
        & \le \frac{1}{2} \sum_{i_0, i_1, \dots, i_{d-1}} \int_{\vx \in \Ce{i_0, i_1, \dots, i_{d-1}}} \tilde{p}_*(\vx) 4 \epsilon \der \vx \\
        & = 2\epsilon,
    \end{aligned}
\end{equation}
where the last inequality follows from Eq.~\eqref{ineq:tilde_p_ratio_lr_bound}.
Hence, the proof is completed.
\end{proof}

\begin{lemma}
    \label{lem:zero_grad_bound}
    Suppose the data distribution $p_*$ satisfy Assumption~\ref{a1}--\ref{a2}, we have
    \begin{equation*}
       \left\|\grad f_*(\vzero)\right\|^2  \le 2Hd + 2H^2m_0
    \end{equation*}
\end{lemma}
\begin{proof}
    We start with the following inequality
    \begin{equation*}
        \begin{aligned}
            & \left\|\grad f_*(\vzero)\right\|^2 = \int_{\vx\in\R^d} p_*(\vx)\left\|\grad f_*(\vzero)\right\|^2 \der \vx\\
            & \le 2 \int_{\vx\in\R^d} p_*(\vx)\left\|\grad f_*(\vx)\right\|^2 \der \vx  + 2 \int_{\vx\in\R^d} p_*(\vx)\left\|\grad f_*(\vzero) - \grad f_*(\vx)\right\|^2 \der \vx \\
            & \le 2Hd + 2H^2\int_{\vx\in \R^d} p_*(\vx) \left\|\vx\right\|^2 \der \vx = 2Hd + 2H^2m_0
        \end{aligned}
    \end{equation*}
    where the second inequality follows from Lemma~\ref{lem:lem11_vempala2019rapid} and Assumption~\ref{a2} and the last inequality follows from Assumption~\ref{a1}.
    Hence, the proof is completed.
\end{proof}

\begin{proof}[Proof of Lemma~\ref{lem:discrete_quantization_gap}]
    The TV distance between the original data distribution $p_*$ and the histogram approximation $\overline{p}_*$ can be written as
    \begin{equation*}
        \TVD{p_*}{\overline{p}_*}\le \TVD{p_*}{\tilde{p}_*} + \TVD{\tilde{p}_*}{\overline{p}_*}.
    \end{equation*}
    Following from Lemma~\ref{lem:histog_dis_comp}, we will have $\TVD{p_*}{\tilde{p}_*} $ by choosing 
    \begin{equation}
        \label{ineq:choice_of_L_main}
        L\ge \sigma\cdot \sqrt{2\ln(2d/\epsilon)}.
    \end{equation}
    Moreover, with the quantization shown in Eq.~\eqref{def:quantize_cell}, it has
    $\TVD{\tilde{p}_*}{\overline{p}_*}\le 2\epsilon$ by choosing 
    \begin{equation}
        \label{ineq:choice_of_l_main}
        l \le (2HL+\|\grad f_*(\vzero)\|)^{-1}\cdot d^{-1/2}\epsilon,
    \end{equation}
    which follows from Lemma~\ref{lem:histog_dis_quan}.
    Combining Eq.~\eqref{ineq:choice_of_L_main} with Eq.~\eqref{ineq:choice_of_l_main}, if we set
    \begin{equation*}
        L=\sigma\cdot \sqrt{2\ln(2d/\epsilon)}\quad \text{and}\quad l \coloneqq  \frac{\epsilon}{2H(L\sqrt{d}+d+\sqrt{dm_0})}
    \end{equation*}
    and $l$ satisfies
    \begin{equation*}
        \begin{aligned}
            & l \le \frac{\epsilon}{\left(2HL + 2\sqrt{Hd} + 2H\sqrt{m_0}\right)\sqrt{d}}\le \frac{\epsilon}{\left(2HL + \sqrt{2Hd + 2H^2m_0}\right)\sqrt{d}}\\
            & \le (2HL+\|\grad f_*(\vzero)\|)^{-1}\cdot d^{-1/2}\epsilon
        \end{aligned}
    \end{equation*}
    where the last inequality follows from Lemma~\ref{lem:zero_grad_bound}.
    That means
    \begin{equation*}
        l= \Omega\left(\left[2H\cdot \left(\sigma\sqrt{2d\ln(2d/\epsilon)} + d + \sqrt{dm_0}\right)\right]^{-1} \cdot \epsilon\right),
    \end{equation*}
    it will have $ \TVD{p_*}{\overline{p}_*}\le 3\epsilon$.
    Hence, the proof is completed.
\end{proof}

\section{Proof of Lemma~\ref{lem:fwd_convergence}}
\label{sec:app_dis_fwd_conv}

To make our analysis clear, we define the variables, the random variables, and the marginal density derived by a specific ordered set $S\subseteq \{0,1,\ldots, d\log_2K-1\}$.
Specifically, we have
\begin{equation*}
    \vy_{S} = \sum_{i=0}^{|S|-1} \ve_i \cdot y_{S_i}\quad \text{and}\quad \rvy_{t,S} = \sum_{i=0}^{|S|-1}\ve_i\cdot \ry_{t, S_i}
\end{equation*}
where there are
\begin{equation*}
    \vy=[y_0, y_1,\ldots , y_{d\log_2K-1}]\quad \text{and}\quad \rvy_t = [\ry_{t,0}, \ry_{t,1},\ldots , \ry_{t, d\log_2K-1}].
\end{equation*}
Suppose $\rvy_t\sim q_t$ The underlying distribution of $\rvy_{t,S}$ is denoted as
\begin{equation*}
    q_{t,S}(\vy_S) = \sum_{\tilde{\vy}\in\gY} q_t(\tilde{\vy})\cdot \vone_{\vy_S}(\tilde{\vy}_S).
\end{equation*}

\begin{lemma}[Modified log-Sobolev inequality for the forward process]
    \label{lem:m_lsi_lemma}
    Suppose the transition rate function $R^\to$ of the CTMC $\{\rvy^\to_t\}_{t=0}^T$ be defined as Eq.~\eqref{eq:fwd_transtion_rate_max}.
    CTMC satisfies modified log-Sobolev inequality with a constant $2$, that is to say, for any $f \in \mathbb{L}_{2}(q_\infty^\to)$, it has 
    \begin{equation*}
        \mathrm{Ent}_{q^\to_\infty}[f]\le \gE(f,\ln f)
    \end{equation*}
    where $\mathrm{Ent}$ and $\gE$ denote the entropy and the Dirichlet functional.
\end{lemma}
\begin{proof}
    We start from the setting of the transition rate matrix of the forward process shown in Eq.~\eqref{eq:fwd_transtion_rate_max}.
    Combining with the Eq.~\eqref{def:infinitesimal_oper}, the infinitesimal generator for the forward process can be obtained, i.e., 
    \begin{equation}
        \gL^\to[f](\vy) = \left<f, R^\to(\cdot,\vy)\right>_{\gY}.
    \end{equation}
    To verify the modified log-Sobolev inequality, we first require to calculate the Dirichlet functional $\gE(f,\ln f)$. 
    Here $\gE$ denotes the Dirichlet functional 
    \begin{equation*}
        \gE(f,g) \coloneqq \int \Gamma(f,g) \der q^\to_\infty,
    \end{equation*}
    where $q_\infty$ denotes the invariant measure of this forward process and $\Gamma$ denotes the carr\'e du champ operator, i.e.,
    \begin{equation*}
        \Gamma(f,g) \coloneqq \frac{1}{2}\left(\gL[f\cdot g] - f\cdot \gL[g] - g\cdot \gL[f]\right).
    \end{equation*}
    Specifically, presenting the transition rate matrix to be a matrix version Eq.~\eqref{eq:fwd_vec}, we have
    \begin{equation*}
        \der \vq^\to_t/\der t = \mR^\to\cdot \vq^\to_t
    \end{equation*}
    Combining the fact $\vone\cdot \mR =\vzero$ and $\mR$ is symmetric, the RHS of the above equation satisfies
    \begin{equation*}
        \mR^\to\cdot  2^{-d\log_2K}\cdot \vone = \vzero,
    \end{equation*}
    which implies the uniform distribution coincides with the invariant measure of $q^\to_\infty$.
    Then, for the Dirichlet functional, it has
    \begin{equation*}
        \begin{aligned}
            & \gE(f,\ln f) = \frac{1}{2}\int \gL[f\cdot \ln f](\vy) - f(\vy)\cdot \gL[\ln f](\vy) - \ln f(\vy)\cdot \gL[f](\vy) \der q_\infty(\vy)\\
            & = \frac{1}{2}\sum_{\vy\in \gY}q_\infty(\vy)\cdot \left[\sum_{\vy^\prime \in \gY} f(\vy^\prime)\ln f(\vy^\prime)\cdot R(\vy^\prime, \vy) - f(\vy)\cdot \sum_{\vy^\prime \in \gY}\ln f(\vy^\prime)\cdot R(\vy^\prime, \vy) \right.\\
            &\qquad\qquad \left. - \ln f(\vy)\cdot \sum_{\vy^\prime\in \gY} f(\vy^\prime)\cdot R(\vy^\prime,\vy) + f(\vy)\cdot \ln f(\vy)\cdot \underbrace{\sum_{\vy^\prime \in \gY} R(\vy^\prime, \vy)}_{=0}\right]\\
            & = \frac{1}{2}\sum_{\vy\in \gY}\sum_{\vy^\prime \in \gY} q_\infty(\vy)\left(f(\vy) - f(\vy^\prime)\right)\cdot R(\vy^\prime, \vy)\cdot (\ln f(\vy) - \ln f(\vy^\prime)).
        \end{aligned}
    \end{equation*}
    Plugging the definition of R into the above equation, we have
    \begin{equation}
        \label{ineq:dirich_simp}
        \begin{aligned}
            \gE(f,\ln f) = & \frac{1}{2}\cdot \sum_{\vy\in\gY}q_\infty(\vy)\cdot \sum_{i=0}^{d\log_2K-1} \sum_{\tilde{y}_i\in \{0,1\}} (f(\vy) - f(\vy+(\tilde{y}_i-y_i) \cdot \ve_i))\\
            &\cdot \left(\ln f(\vy) - \ln f(\vy+(\tilde{y}_i-y_i)\cdot \ve_i)\right).
        \end{aligned}
    \end{equation}
    Then, we consider $\mathrm{Ent}_{q^\to_\infty}[f]$, which satisfies
    \begin{equation}
        \label{ineq:entropy_upb}
        \begin{aligned}
            & \mathrm{Ent}_{q^\to_\infty}[f] = \E_{\rvy\sim q^\to_\infty}\left[f(\rvy)\ln f(\rvy)\right] - \E_{\rvy\sim q^\to_\infty}[f(\rvy)]\ln\left(\E_{\rvy\sim q^\to_\infty}[f(\rvy)]\right)\\
            & \le \sum_{i=0}^{d\log_2K-1}\E_{\rvy_{[0:i-1, i+1:d\log_2K-1]}}\left[\underbrace{\E_{\ry_i}\left[f(\rvy)\ln f(\rvy)\right] - \E_{\ry_i}\left[f(\rvy)\right]\ln\left(\E_{\ry_i}\left[f(\rvy)\right]\right)}_{\text{Term 1}}\right].
        \end{aligned}
    \end{equation}
    due to the sub-additivity of the entropy, i.e., Lemma~\ref{lem:thm4_10_boucheron2003concentration}.
    Term 1 of Eq.~\eqref{ineq:entropy_upb} satisfies
    \begin{equation*}
        \begin{aligned}
            \text{Term 1} = & \sum_{y_i\in \{0,1\}} q^\to_{\infty,i}(y_i)\cdot f(\vy_{0:i-1}, y_i, \vy_{i+1:d\log_2K-1})\ln f(\vy_{0:i-1}, y_i, \vy_{i+1:d\log_2K-1}) \\
            & - \sum_{y_i\in \{0,1\}} q^\to_{\infty,i}(\vy_i) f(\vy_{0:i-1}, y_i, \vy_{i+1:d\log_2K-1})\\
            &\qquad \cdot \ln\left(\sum_{\tilde{y}_i\in \{0,1\}} q^\to_{\infty, i}(\tilde{y}_i) f(\vy_{0:i-1}, \tilde{y}_i, \vy_{i+1:d\log_2K-1})\right)\\
            \le & \sum_{y_i\in \{0,1\}} q^\to_{\infty,i}(y_i)\cdot f(\vy_{0:i-1}, y_i, \vy_{i+1:d\log_2K-1})\\
            & \quad \cdot \sum_{\tilde{y}_i\in \{0,1\}}\left[\frac{\ln f(\vy_{0:i-1}, y_i, \vy_{i+1:d\log_2K-1})}{2} - \frac{\ln f(\vy_{0:i-1}, \tilde{y}_i, \vy_{i+1:d\log_2K-1})}{2}\right]\\
            \le & \frac{1}{2}\sum_{y_i, \tilde{y}_i\in\{0,1\}} q^\to_{\infty,i}(y_i) \cdot \left(f(\vy_{0:i-1}, y_i, \vy_{i+1:d\log_2K-1}) - f(\vy_{0:i-1}, \tilde{y}_i, \vy_{i+1:d\log_2K-1})\right)\\
            & \quad \cdot \left(\ln f(\vy_{0:i-1}, y_i, \vy_{i+1:d\log_2K-1}) - \ln f(\vy_{0:i-1}, \tilde{y}_i, \vy_{i+1:d\log_2K-1})\right),
        \end{aligned}
    \end{equation*}
    where the first inequality follows from the concavity of the logarithm function, and the last inequality follows from 
    \begin{equation*}
        \begin{aligned}
            & \sum_{\vy, \tilde{\vy}} \cdot f(\vy)\cdot (\ln f(\vy) - \ln f(\tilde{\vy})) = \sum_{\tilde{\vy}} f(\tilde{\vy})\cdot (\ln f(\tilde{\vy}) - \ln f(\vy))\\
            & = \frac{1}{2}\sum_{\vy, \tilde{\vy}} \cdot (f(\vy)-f(\tilde{\vy}))\cdot (\ln f(\vy) - \ln f(\tilde{\vy}))
        \end{aligned}
    \end{equation*}
    and $q^\to_\infty(\cdot)$ is a constant function.
    Then, plugging this inequality into Eq.~\eqref{ineq:entropy_upb}, we have
    \begin{equation}
        \label{ineq:ent_upb}
        \begin{aligned}
            & \mathrm{Ent}_{q^\to_\infty}[f]  \le \frac{1}{2}\cdot \sum_{i=0}^{d\log_2K - 1}\left[\sum_{\vy_{[0:i-1, i+1:d\log_2K-1]}} q^\to_{\infty, [0:i-1, i+1:d\log_2K-1]}(\rvy_{[0:i-1, i+1:d\log_2K-1]})\right.\\
            & \qquad \left. \sum_{y_i} q^\to_{\infty,i}(y_i) \sum_{\tilde{\vy}_i} \left(f(\vy_{0:i-1}, y_i, \vy_{i+1:d\log_2K-1}) - f(\vy_{0:i-1}, \tilde{y}_i, \vy_{i+1:d\log_2K-1})\right) \right.\\
            &\qquad \left. \cdot \left(\ln f(\vy_{0:i-1}, y_i, \vy_{i+1:d\log_2K-1}) - \ln f(\vy_{0:i-1}, \tilde{y}_i, \vy_{i+1:d\log_2K-1})\right) \right]\\
            & = \frac{1}{2}\cdot \sum_{\vy} q^\to_{\infty}(\vy)\cdot \sum_{i=0}^{d\log_2K-1} \sum_{\tilde{\vy}_i} (f(\vy) - f(\vy+(\tilde{y}_i-y_i)\cdot \ve_i))\\
            & \qquad \cdot \left(\ln f(\vy) - \ln f(\vy+(\tilde{y}_i-y_i)\cdot \ve_i)\right)
        \end{aligned}
    \end{equation}
    Comparing Eq.~\eqref{ineq:ent_upb} and Eq.~\eqref{ineq:dirich_simp}, it satisfies
    \begin{equation*}
        \mathrm{Ent}_{q^\to_\infty}[f]\le \frac{C_{\text{LSI}}}{2}\cdot \gE(f,\ln f)
    \end{equation*}
    by choosing $C_\text{LSI} = 2$.
\end{proof}

\begin{proof}[Proof of Lemma~\ref{lem:fwd_convergence}]
    We investigate the dynamic of KL divergence between $q_t^\to$ and $q_\infty^\to$ in the forward process.
    Specifically, we have
    \begin{equation*}
        \begin{aligned}
            & \frac{\der \KL{q^\to_t}{q^\to_\infty}}{\der t} =\sum_{\vy\in \gY} \frac{\der q^\to_t(\vy)}{\der t} \cdot \ln \frac{q^\to_t(\vy)}{q^\to_\infty(\vy)} = \sum_{\vy\in \gY} \ln \frac{q^\to_t(\vy)}{q^\to_\infty(\vy)}\left(\sum_{\vy_0 \in \gY}R(\vy, \vy_0)\cdot q^\to_t(\vy_0)\right)\\
            & = \sum_{\vy_0} q^\to_\infty(\vy_0)\cdot \frac{q^\to_t(\vy_0)}{q^\to_\infty(\vy_0)}\cdot\sum_{\vy} \ln\frac{q^\to_t(\vy)}{q^\to_\infty(\vy)}\cdot R^\to(\vy, \vy_0)\\
            & = \sum_{\vy_0} q^\to_\infty(\vy_0)\cdot \frac{q^\to_t(\vy_0)}{q^\to_\infty(\vy_0)}\cdot \gL[\ln \frac{q^\to_t}{q^\to_\infty}](\vy^\prime) = - \gE\left(\frac{q^\to_t}{q^\to_\infty}, \ln \frac{q^\to_t}{q^\to_\infty}\right)
        \end{aligned}
    \end{equation*}
    Due to Lemma~\ref{lem:m_lsi_lemma}, we have
    \begin{equation*}
        \frac{\der \KL{q^\to_t}{q^\to_\infty}}{\der t} = - \gE\left(\frac{q^\to_t}{q^\to_\infty}, \ln \frac{q^\to_t}{q^\to_\infty}\right)\le \mathrm{Ent}_{q_\infty}\left[\frac{q^\to_t}{q_\infty^\to}\right]  = -\KL{q^\to_t}{q^\to_\infty}.
    \end{equation*}
    According to the Gronwall's theorem, we have
    \begin{equation*}
        \KL{q^\to_t}{q^\to_\infty}\le e^{-t}\cdot \KL{q^\to_0}{q^\to_\infty}.
    \end{equation*}
    Combining with the following initialization error bound, 
    \begin{equation*}
        \KL{q^\to_0}{q^\to_\infty} = \sum_{\vy\in \gY} q^\to_0(\vy)\ln \frac{q^\to_0(\vy)}{2^{-d\log_2K}}\le d\log_2K.
    \end{equation*}
    Hence, the proof is completed.
\end{proof}

\section{Supplementary Proofs for the Discrete Reverse Process}

\subsection{Proof of Lemma~\ref{lem:out_degree_rate_wrt_time}}
\label{sec:out_degree_rate_wrt_time}

\begin{proof}[Proof of Lemma~\ref{lem:out_degree_rate_wrt_time} (adapted from Proposition 5 of~\cite{chen2024convergence})]
    Suppose the transition rate function $R^\to$ of the CTMC $\{\rvy^\to_t\}_{t=0}^T$ be defined as Eq.~\eqref{eq:fwd_transtion_rate_max}, 
    the marginal distribution at time $t$ can be written as
    \begin{equation*}
        q^\to_t(\vy) = \sum_{\vy_0\in \gY} q_0^\to(\vy_0)\cdot q_{t|0}^\to(\vy|\vy_0). 
    \end{equation*}
    Define the plus operator as follows
    \begin{equation*}
        \vy \oplus \ve_i = \left[y_0, y_1,\ldots, y_{i-1}, (y_i+1)\  \mathrm{mod}\ 2, y_{i+1},\ldots, y_{d\log_2K -1}\right],
    \end{equation*}
    then we have
    \begin{equation*}
        \begin{aligned}
            \frac{q_t^\to(\vy\oplus \ve_i)}{q^\to_t(\vy)} = \frac{\sum_{\vy_0\in \gY}q_0^\to(\vy_0)\cdot q_{t|0}^\to(\vy+\ve_i|\vy_0)}{\sum_{\vy_0\in \gY}q_0^\to(\vy_0)\cdot q_{t|0}^\to(\vy|\vy_0)} = \frac{\sum_{\vy_0\in \gY}q_0^\to(\vy_0)\cdot q^\to_{t|0}(\vy|\vy_0)\cdot \frac{q_{t|0}^\to(\vy+\ve_i|\vy_0)}{q^\to_{t|0}(\vy|\vy_0)} }{\sum_{\vy_0\in \gY}q_0^\to(\vy_0)\cdot q_{t|0}^\to(\vy|\vy_0)}.
        \end{aligned}
    \end{equation*}
    According to Bayes Theorem, we have
    \begin{equation*}
        q^\to_{0|t}(\vy_0|\vy)\cdot q^\to_{t}(\vy) = q_{t|0}^\to(\vy|\vy_0)\cdot q^\to_0(\vy_0)\quad \Leftrightarrow\quad q^\to_{0|t}(\vy_0|\vy)\propto  q_{t|0}^\to(\vy|\vy_0)\cdot q^\to_0(\vy_0),
    \end{equation*}
    which implies
    \begin{equation*}
        \frac{q_t^\to(\vy\oplus \ve_i)}{q^\to_t(\vy)} = \E_{\rvy_0\sim q^\to_{0|t}(\cdot|\vy)}\left[\frac{q_{t|0}^\to(\vy+\ve_i|\vy_0)}{q^\to_{t|0}(\vy|\vy_0)} \right].
    \end{equation*}
    With Lemma~\ref{lem:fwd_trans_ker}, we have
    \begin{equation*}
        \frac{q_{t|0}^\to(\vy+\ve_i|\vy_0)}{q^\to_{t|0}(\vy|\vy_0)} = \frac{1+(-1)^{|(y_i + 1 - y_{0,i})\ \mathrm{mod}\ 2|}\cdot e^{-2t}}{1+(-1)^{|(y_i - y_{0,i})\ \mathrm{mod}\ 2|}\cdot e^{-2t}} \le \frac{1+e^{-2t}}{1-e^{-2t}},
    \end{equation*}
    which means 
    \begin{equation*}
        \frac{q_t^\to(\vy\oplus \ve_i)}{q^\to_t(\vy)}\le \frac{1+e^{-2t}}{1-e^{-2t}}\le 1+t^{-1}.
    \end{equation*}
    Therefore, if we consider the transition rate matrix of the reverse process, i.e.,
    \begin{equation*}
        R_t^\gets(\vy^\prime, \vy) \coloneqq   R^\to(\vy, \vy^\prime)\cdot\frac{q^\gets_t(\vy^\prime)}{q^\gets_t(\vy)}
    \end{equation*}
    provided by Eq~\eqref{eq:rev_func}, it has
    \begin{equation*}
        \sum_{\vy^\prime\not=\vy} R^\gets_t(\vy^\prime,\vy) = \sum_{i=0}^{d\log_2K-1} \frac{q_t^\gets(\vy \oplus \ve_i)}{q^\gets_t(\vy)} =  \sum_{i=0}^{d\log_2K-1} \frac{q^\to_{T-t}(\vy \oplus \ve_i)}{q^\to_{T-t}(\vy)}\le (d\log_2K)\cdot (1+(T-t)^{-1}).
    \end{equation*}
    Hence, the proof is completed.
\end{proof}

\subsection{Proof of Theorem~\ref{thm:main_thm}}
\label{sec:proof_main_thm}

\begin{table*}[t]
    \centering
    \begin{tabular}{ccccc}
    \toprule
     Results & Algorithm & Assumptions & Early Stopping & Complexity (for TV) \\
     \midrule
    \cite{chen2024convergence} & Uniformization & \ref{a4}, \eqref{ineq:estimated_score_bound} & Yes  & $\tilde{O}(d)$\\
    \midrule
     \cite{zhang2024convergence} & Euler-Method & \ref{a4}  & Yes & $\tilde{\mathcal{O}}(d^{4/3}\epsilon^{-4/3})$\\
     \midrule
     \cite{ren2025fast} &  $\tau$--leaping & \ref{a4} & Yes & $\tilde{\mathcal{O}}(d\epsilon^{-1})$ \\
     \midrule
     ours & Truncated-Uniformization & \ref{a4} & Yes & $\tilde{\mathcal{O}}(d)$\\
     \bottomrule 
    \end{tabular}
    \caption{Comparison with prior discrete inference algorithm. Stopping time will be $T-\epsilon/d$  to guarantee the TV convergence.} 
    \label{tab:comp_old_dis}
\end{table*}

The ultimate target of Alg.~\ref{alg:uni_inf} is to generate sample $\hat{\rvx}$ and require its underlying distribution $\hat{p}$ to be close to the continuous data distribution $p_*$. 
However, Alg.~\ref{alg:uni_inf} can be divided into two parts:
\begin{enumerate}
    \item \textbf{Truncated Uniformization}: Generate a discrete sample following $\hat{q}_{T-\delta} = \hat{q}_{t_W}$ which approximates $q_*$, which is from Step.~\ref{step:infer_with_trunc_uni_start} to Step.~\ref{step:approximate_discrte_sample}.\\
    \item Mapping the generated discrete data to the corresponding cell in Euclidean space and uniformly drawing a sample from the cell, which is from Step.~\ref{step:from_discrete_to_continuous}
\end{enumerate}

All the following notations correspond to those mentioned in Alg.~\ref{alg:uni_inf}.

\begin{lemma}
    \label{lem:timestamp_prop}
    Suppose we have a timestamp sequence satisfying 
    \begin{equation*}
        t_0 = 0 \quad \text{and}\quad  t_{w+1} - t_w = 0.5\cdot (T-t_{w+1}),
    \end{equation*}
    then we know the sequence $\{t_w\}_{w=0}^W$ is strict increasing and $t_W<T$ for any $W$.
\end{lemma}
\begin{proof}
According to the timestamp setting, i.e.,
\begin{equation*}
    t_0 = 0 \quad \text{and}\quad  t_{w+1} - t_w = 0.5\cdot (T-t_{w+1}),
\end{equation*}
solve for $t_{k+1}$, we have
\begin{equation*}
    t_{w+1} = \frac{0.5T + t_w}{1.5}  = \frac{T + 2t_w}{3}.
\end{equation*}
Then, we consider the difference:
\begin{equation*}
    t_{w+1} - t_w = \frac{T + 2\,t_w}{3} - t_w = \frac{T + 2t_w - 3t_w}{3}= \frac{T - t_w}{3}.
\end{equation*}
If $T - t_w > 0$, then we have
\begin{equation*}
    t_{w+1} - t_w =\frac{T - t_w}{3} > 0,
\end{equation*}
which shows \(t_{w+1} > t_w\). Thus, as long as \(t_w < T\), the sequence is strictly increasing.

Moreover, due to the fact $t_0  =0 < T$, we can prove that \(t_w < T\) for all \(w\).
Specifically, assume \(t_w < T\); then
\begin{equation*}
    t_{w+1} = \frac{T + 2t_w}{3} < \frac{T + 2T}{3} =T.
\end{equation*}
Therefore, \(t_{w+1} < T\) as well, completing the induction. 
Hence \(t_w\) remains below \(T\) for all \(w\), and the sequence \(\{t_w\}\) is strictly increasing.
\end{proof}

\begin{lemma}
    \label{lem:ideal_transtion_event_num}
    Suppose the reverse process is divided into $W$ segments with endpoints $\{t_w\}_{w=0}^W$ satisfying
    \begin{equation*}
        t_0 = 0, \quad  t_{w+1} - t_w = 0.5\cdot (T-t_{w+1})\quad \text{and}\quad t_W= T-\delta,
    \end{equation*}
    if we set 
    \begin{equation*}
        \beta_{t_w}\coloneqq 2d \log_2K/{\min\{1, T-t_w\}}
    \end{equation*}
    then we have
    \begin{equation*}
        \sum_{k=1}^W \beta_{t_w}\cdot (t_w-t_{w-1}) \le 2d\log_2K\cdot \left(T+\ln(1/\delta)\right)
    \end{equation*}
\end{lemma}
\begin{proof}[Adapted from Theorem 6 of~\cite{chen2024convergence}]
    Suppose there exist time steps $t_0, t_1, \dots, t_W$ such that $T - t_w = s_w$ for each $w=0,\dots,W$. 
    According to Lemma~\ref{lem:timestamp_prop}, we know $\{t_w\}_{w=0}^W$ is a increasing sequence, if we set
    \begin{equation*}
        s_w\coloneqq T - t_w,
    \end{equation*}
    then it can be expected that $s_0> s_1 > \dots > s_W \ge \delta > 0$.
    According to the choice of $\beta_w$, it has
    \begin{equation*}
        \beta_w = \frac{Cd\log_2K}{\min(1,\, s_w)}, \quad \text{and} \quad s_{w-1} - s_w > 0.
    \end{equation*}
    
    For $w$ such that $\delta \le s_w < 1$, notice that $\min(1, s_w) = s_w$, we have $\beta_w=Cd\log_2K/s_w$ and 
    \begin{equation*}
        \sum_{\substack{w:\delta \le s_w < 1}}\beta_w\cdot (t_w - t_{w-1})  = \sum_{\substack{w:\delta \le s_w < 1}} \beta_w\,(s_{w-1} - s_w)\;=\; \sum_{\substack{w:\delta \le s_w < 1}} \frac{C d\log_2K}{s_w}\,(s_{w-1} - s_w).
    \end{equation*}
    Because $1/s$ is a decreasing function for $s>0$, we have
    \begin{equation*}
        \frac{1}{s_w} \le \frac{1}{s}  \quad \text{for all}\ s \in [s_w, s_{w-1}],
    \end{equation*}
    which implies
    \begin{equation*}
        \frac{C d}{s_w}\,(s_{w-1} - s_w)\le  Cd\log_2K \int_{s_w}^{s_{w-1}} \frac{1}{s} \der s.
    \end{equation*}
    Hence,
    \begin{equation*}
        \sum_{\substack{w:\delta \le s_w < 1}} \frac{C d\log_2K}{s_w}\,(s_{w-1} - s_w) \le Cd\log_2K \sum_{\substack{w:\delta \le s_w < 1}} \int_{s_w}^{s_{w-1}} \!\!\frac{1}{s}\,ds = Cd\log_2K \int_{\delta}^{1} \frac{1}{s}\,ds.
    \end{equation*}
    Evaluating the integral on the right gives
    \begin{equation*}
        C d\log_2K \int_{\delta}^{1} \frac{1}{s} \der s = C d\log_2K \bigl[\ln(s)\bigr]_{\delta}^{1} = C d\log_2K \ln(1/\delta).
    \end{equation*}
    Therefore, we have established the exact upper bound
    \begin{equation*}
        \sum_{\substack{k:\delta \le s_k < 1}} \lambda_k\,(s_{k-1} - s_k) \le  C d\log_2K \ln(1/\delta).
    \end{equation*}
    For $s_w\ge 1$, , notice that $\min(1, s_w) = 1$, we have $\beta_w=Cd\log_2K$.
    \begin{equation*}
        \begin{aligned}
            & \sum_{\substack{w:1 \le s_w \le T }}\beta_w\cdot (t_w - t_{w-1})  = \sum_{\substack{w:1 \le s_w \le T}} \beta_w\,(s_{w-1} - s_w)\\
            & = \sum_{\substack{w:1 \le s_w \le T}} C d\log_2K \cdot (s_{w-1} - s_w) \le Cd\log_2K\cdot (T-1).
        \end{aligned}
    \end{equation*}
    Combining the two parts, we have
    \begin{equation*}
        \begin{aligned}
            &\sum_{w=1}^W \beta_w\cdot (t_w-t_{w-1}) = \sum_{w=1}^W \beta_w\cdot (s_{w-1}-s_{w})\\
            & = \sum_{\substack{w:\delta \le s_w < 1}} \beta_w\,(s_{w-1} - s_w)+  \sum_{\substack{w:1 \le s_w \le T}} \beta_w\,(s_{w-1} - s_w)\le Cd\log_2K\cdot \left(T+\ln(1/\delta)\right).
        \end{aligned}
    \end{equation*}
    Hence, the proof is completed.
\end{proof}

\begin{lemma}
    \label{lem:lim_deltat_ln}
    Following the notations shown in Section~\ref{sec:pre}, we have
    \begin{equation*}
         \lim_{\Delta t\rightarrow 0}\left[\Delta t^{-1} \cdot \ln \frac{1-\sum_{\vy^\prime\not=\vy}q^\gets_{t+\Delta t|t}(\vy^\prime|\vy)}{1-\sum_{\vy^\prime\not=\vy}\hat{q}_{t+\Delta t|t}(\vy^\prime|\vy)}\right] = \hat{R}_t(\vy) - R^\gets_t(\vy).
    \end{equation*}
\end{lemma}
\begin{proof}
    Since we have required $\Delta t\rightarrow 0$, that is to say
    \begin{equation*}
        \hat{q}_{t+\Delta t|t}(\vy^\prime|\vy)\rightarrow \hat{q}_{t|t}(\vy^\prime|\vy)=0\quad \text{and}\quad q^\gets_{t+\Delta t|t}(\vy^\prime|\vy)\rightarrow q^\gets_{t|t}(\vy^\prime|\vy)=0\quad \forall \vy^\prime\not=\vy,
    \end{equation*}
    which automatically makes 
    \begin{equation*}
        \left|\frac{\sum_{\vy^\prime\not=\vy}\left(\hat{q}_{t+\Delta t|t}(\vy^\prime|\vy) - q^\gets_{t+\Delta t|t}(\vy^\prime|\vy)\right)}{1-\sum_{\vy^\prime\not=\vy}\hat{q}_{t+\Delta t|t}(\vy^\prime|\vy)}\right|\le \frac{1}{2}<1.
    \end{equation*}
    Under this condition, we have
    \begin{equation*}
        \begin{aligned}
            & \ln \frac{1-\sum_{\vy^\prime\not=\vy}q^\gets_{t+\Delta t|t}(\vy^\prime|\vy)}{1-\sum_{\vy^\prime\not=\vy}\hat{q}_{t+\Delta t|t}(\vy^\prime|\vy)} =\ln\left[ 1 + \frac{\sum_{\vy^\prime\not=\vy} \left(\hat{q}_{t+\Delta t|t}(\vy^\prime|\vy)-q^\gets_{t+\Delta t|t}(\vy^\prime|\vy)\right)}{1-\sum_{\vy^\prime\not=\vy}\hat{q}_{t+\Delta t|t}(\vy^\prime|\vy)}\right]\\
            & = \sum_{i=1}^\infty \frac{(-1)^{i+1}}{i}\cdot \left[\frac{\sum_{\vy^\prime\not=\vy} \left(\hat{q}_{t+\Delta t|t}(\vy^\prime|\vy)-q^\gets_{t+\Delta t|t}(\vy^\prime|\vy)\right)}{1-\sum_{\vy^\prime\not=\vy}\hat{q}_{t+\Delta t|t}(\vy^\prime|\vy)}\right]^i,
        \end{aligned}
    \end{equation*}
    which implies (with the dominated convergence theorem)
    \begin{equation*}
        \begin{aligned}
             & \lim_{\Delta t\rightarrow 0}\left[\Delta t^{-1} \cdot \ln \frac{1-\sum_{\vy^\prime\not=\vy}q^\gets_{t+\Delta t|t}(\vy^\prime|\vy)}{1-\sum_{\vy^\prime\not=\vy}\hat{q}_{t+\Delta t|t}(\vy^\prime|\vy)}\right]\\
             & = \sum_{i=1}^\infty \frac{(-1)^{i+1}}{i}\cdot \lim_{\Delta t\rightarrow 0} \frac{\sum_{\vy^\prime\not=\vy} \left(\hat{q}_{t+\Delta t|t}(\vy^\prime|\vy)-q^\gets_{t+\Delta t|t}(\vy^\prime|\vy)\right)}{\Delta t}\\
             & \quad \cdot \lim_{\Delta t \rightarrow 0}\frac{\left(\sum_{\vy^\prime\not=\vy} \left(\hat{q}_{t+\Delta t|t}(\vy^\prime|\vy)-q^\gets_{t+\Delta t|t}(\vy^\prime|\vy)\right)\right)^{i-1}}{\left(1-\sum_{\vy^\prime\not=\vy}\hat{q}_{t+\Delta t|t}(\vy^\prime|\vy)\right)^i}.
        \end{aligned}
    \end{equation*}
    Only when $i=1$, we have
    \begin{equation*}
        \lim_{\Delta t \rightarrow 0}\frac{\left(\sum_{\vy^\prime\not=\vy} \left(\hat{q}_{t+\Delta t|t}(\vy^\prime|\vy)-q^\gets_{t+\Delta t|t}(\vy^\prime|\vy)\right)\right)^{i-1}}{\left(1-\sum_{\vy^\prime\not=\vy}\hat{q}_{t+\Delta t|t}(\vy^\prime|\vy)\right)^i}=1,
    \end{equation*}
    otherwise it will be equivalent to $0$.
    Therefore, we have
    \begin{equation*}
        \begin{aligned}
            &\lim_{\Delta t\rightarrow 0}\left[\Delta t^{-1} \cdot \ln \frac{1-\sum_{\vy^\prime\not=\vy}q^\gets_{t+\Delta t|t}(\vy^\prime|\vy)}{1-\sum_{\vy^\prime\not=\vy}\hat{q}_{t+\Delta t|t}(\vy^\prime|\vy)}\right] = \lim_{\Delta t\rightarrow 0} \frac{\sum_{\vy^\prime\not=\vy} \left(\hat{q}_{t+\Delta t|t}(\vy^\prime|\vy)-q^\gets_{t+\Delta t|t}(\vy^\prime|\vy)\right)}{\Delta t}\\
            & = \sum_{\vy^\prime\not=\vy} \left(\hat{R}_t(\vy^\prime,\vy) - R^\gets_t(\vy^\prime,\vy)\right) = \hat{R}_t(\vy) - R^\gets_t(\vy).
        \end{aligned}
    \end{equation*}
    Hence, the proof is completed.
\end{proof}

\begin{lemma}
    \label{thm:unif_linear_convergence}
    Suppose Assumption~\ref{a4} holds, if we conduct the reverse process as Alg.~\ref{alg:uni_inf}, then we have
    \begin{equation*}
        \KL{q^\gets_{T-\delta}}{\hat{q}_{T-\delta}}\le \KL{q^\gets_{0}}{\hat{q}_{0}}+(T-\delta)\epsilon^2_{\text{score}} 
    \end{equation*}
\end{lemma}
\begin{proof}
    We start from the dynamic of KL divergence with the time growth in the reverse process, i.e.,
    \begin{equation*}
        \begin{aligned}
            &\frac{\der \KL{q^\gets_t}{\hat{q}_t} }{\der t} = \lim_{\Delta t \rightarrow 0}\left[\frac{\KL{q^\gets_{t+\Delta t}}{\hat{q}_{t+\Delta t}} - \KL{q^\gets_t}{\hat{q}_t}}{\Delta t}\right]\\
            & \le \lim_{\Delta t\rightarrow 0}\left[\frac{\E_{\rvy\sim q^\gets_t}\left[\KL{q^\gets_{t+\Delta t|t}(\cdot|\rvy)}{\hat{q}_{t+\Delta t|t}(\cdot|\rvy)}\right]}{\Delta t}\right]
        \end{aligned}
    \end{equation*}
    where the inequality follows from the chain rule of KL divergence, i.e., Lemma~\ref{}.
    Under this condition, we have 
    \begin{equation}
        \label{ineq:dKL_ineq_init}
        \begin{aligned}
            \frac{\der \KL{q^\gets_t}{\hat{q}_t} }{\der t} \le \sum_{\vy\in \gY} q^\gets_t(\vy) \cdot \underbrace{\lim_{\Delta t\rightarrow 0}\left[\frac{\KL{q^\gets_{t+\Delta t|t}(\cdot|\vy)}{\hat{q}_{t+\Delta t|t}(\cdot|\vy)}}{\Delta t}\right]}_{\text{Term 1}}.
        \end{aligned}
    \end{equation}
    For each $\vy\in\gY$, we focus on Term 1 of Eq.~\eqref{ineq:dKL_ineq_init}, and have
    \begin{equation}
        \label{eq:term1_equ}
        \begin{aligned}
            \text{Term 1} & = \lim_{\Delta t\rightarrow 0}\left[\Delta t^{-1}\cdot \sum_{\vy^\prime\in\gY} q^\gets_{t+\Delta t|t}(\vy^\prime|\vy)\cdot \ln \frac{q^\gets_{t+\Delta t|t}(\vy^\prime|\vy)}{\hat{q}_{t+\Delta t|t}(\vy^\prime|\vy)}\right]\\
            & = \underbrace{\lim_{\Delta t\rightarrow 0}\left[\sum_{\vy^\prime\not=\vy}\frac{q^\gets_{t+\Delta t|t}(\vy^\prime|\vy)}{\Delta t} \cdot \ln \frac{q^\gets_{t+\Delta t|t}(\vy^\prime|\vy)}{\hat{q}_{t+\Delta t|t}(\vy^\prime|\vy)} \right]}_{\text{Term 1.1}}+\\
            &\quad \underbrace{\lim_{\Delta t\rightarrow 0}\left[\Delta t^{-1}\cdot \left(1-\sum_{\vy^\prime\not=\vy}q^\gets_{t+\Delta t|t}(\vy^\prime|\vy)\right) \cdot \ln \frac{1-\sum_{\vy^\prime\not=\vy}q^\gets_{t+\Delta t|t}(\vy^\prime|\vy)}{1-\sum_{\vy^\prime\not=\vy}\hat{q}_{t+\Delta t|t}(\vy^\prime|\vy)} \right]}_{\text{Term 1.2}}.
        \end{aligned}
    \end{equation}
    For Term 1.1, we have
    \begin{equation}
        \begin{aligned}
            \label{ineq:term1.1_equ}
            \text{Term 1.1} = & \sum_{\vy^\prime\not=\vy} \lim_{\Delta t\rightarrow 0}\left[\frac{q^\gets_{t+\Delta t|t}(\vy^\prime|\vy)}{\Delta t}\right]\cdot \lim_{\Delta t\rightarrow 0}\left[ \ln \frac{q^\gets_{t+\Delta t|t}(\vy^\prime|\vy)}{\hat{q}_{t+\Delta t|t}(\vy^\prime|\vy)}\right]\\
            = & \sum_{\vy^\prime\not=\vy} R^\gets_{t}(\vy^\prime,\vy)\cdot \ln\left[\lim_{\Delta t\rightarrow 0}\left(\frac{q^\gets_{t+\Delta t|t}(\vy^\prime|\vy)}{\Delta t}\cdot \frac{\Delta t}{\hat{q}_{t+\Delta t|t}(\vy^\prime|\vy)}\right)\right]\\
            = & \sum_{\vy^\prime\not=\vy} R^\gets_t(\vy^\prime,\vy) \cdot \ln\frac{R_t^\gets(\vy^\prime,\vy)}{\hat{R}_t(\vy^\prime,\vy)}, 
        \end{aligned}
    \end{equation}
    where the second equation follows from the composition rule of the limit calculation.
    For Term 1.2, we have
    \begin{equation}
        \begin{aligned}
            \label{ineq:term1.2_equ}
            & \text{Term 1.2} = \lim_{\Delta t\rightarrow 0} \left[1-\sum_{\vy^\prime\not=\vy}q^\gets_{t+\Delta t|t}(\vy^\prime|\vy)\right]\cdot \lim_{\Delta t\rightarrow 0}\left[\Delta t^{-1} \cdot \ln \frac{1-\sum_{\vy^\prime\not=\vy}q^\gets_{t+\Delta t|t}(\vy^\prime|\vy)}{1-\sum_{\vy^\prime\not=\vy}\hat{q}_{t+\Delta t|t}(\vy^\prime|\vy)}\right]\\
            & = \sum_{\vy^\prime\not=\vy}\left(\hat{R}_t(\vy^\prime,\vy) - R^\gets_t(\vy^\prime,\vy)\right) = \hat{R}_t(\vy) - R^\gets_t(\vy)
        \end{aligned}
    \end{equation}
    where the first inequality follows from Lemma~\ref{lem:lim_deltat_ln}.
    Plugging Eq.~\eqref{ineq:term1.1_equ}, Eq.~\eqref{ineq:term1.2_equ} and Eq.~\eqref{eq:term1_equ}, into Eq.~\eqref{ineq:dKL_ineq_init} we have
    \begin{equation}
        \label{ineq:dKL_ineq_mid}
        \begin{aligned}
            & \frac{\der \KL{q^\gets_t}{\hat{q}_t} }{\der t} \le \sum_{\vy\in \gY} q^\gets_t(\vy) \cdot \left(\sum_{\vy^\prime\not=\vy} R^\gets_t(\vy^\prime,\vy) \cdot \ln\frac{R_t^\gets(\vy^\prime,\vy)}{\hat{R}_t(\vy^\prime,\vy)} + \hat{R}_t(\vy) - R^\gets_t(\vy)\right).
        \end{aligned}
    \end{equation}
    For any $\vy\in \gY$, we have
    \begin{equation}
        \label{eq:dKL_gap_ScoEst_err}
        \begin{aligned}
            &\sum_{\vy^\prime\not=\vy} R^\gets_t(\vy^\prime,\vy) \cdot \ln\frac{R_t^\gets(\vy^\prime,\vy)}{\hat{R}_t(\vy^\prime,\vy)} + \hat{R}_t(\vy) - R^\gets_t(\vy)\\
            & = \sum_{\vy^\prime\not=\vy} R^\gets_t(\vy^\prime,\vy)\ln \frac{R^\gets_t(\vy^\prime,\vy)}{\tilde{R}_t(\vy^\prime,\vy)} + \tilde{R}_t(\vy) - R^\gets_t(\vy)\\
            &\quad \underbrace{+\sum_{\vy^\prime\not=\vy}R^\gets_t(\vy^\prime,\vy)\ln\frac{\tilde{R}_t(\vy^\prime,\vy)}{\hat{R}_t(\vy^\prime,\vy)} + \hat{R}_t(\vy) - \tilde{R}_t(\vy)}_{\text{Term 2}}.
        \end{aligned}
    \end{equation}
    When $\tilde{R}_t(\vy)\le \beta_t$, we have 
    \begin{equation*}
        \hat{R}_(\vy^\prime, \vy) = \tilde{R}_t(\vy^\prime,\vy)\quad \text{and}\quad \hat{R}(\vy)=\sum_{\vy^\prime\not=\vy}\hat{R}(\vy^\prime,\vy) = \sum_{\vy^\prime\not=\vy}\tilde{R}(\vy^\prime,\vy) = \tilde{R}(\vy)
    \end{equation*}
    which implies $\text{Term 2}=0$ in Eq.~\eqref{eq:dKL_gap_ScoEst_err}.
    Otherwise, we have
    \begin{equation*}
        \frac{\hat{R}_(\vy^\prime, \vy)}{\tilde{R}_t(\vy^\prime,\vy)} = \frac{\beta_t}{\tilde{R}_t(\vy)} \quad\text{and}\quad \frac{\hat{R}_(\vy)}{\tilde{R}_t(\vy)}=\frac{\beta_t}{\tilde{R}_t(\vy)},
    \end{equation*}
    which implies
    \begin{equation*}
        \begin{aligned}
            &\text{Term 2} = \sum_{\vy^\prime\not=\vy}R_t^\gets(\vy^\prime,\vy)\cdot \ln \frac{\tilde{R}_t(\vy)}{\beta_t} + \beta_t - \tilde{R}_t(\vy)\\
            & = R^\gets_t(\vy)\cdot \ln \left[1+ \frac{\tilde{R}_t(\vy)-\beta_t}{\beta_t}\right]+\beta_t - \tilde{R}_t(\vy)\le \beta_t\cdot \left[\frac{\tilde{R}_t(\vy)-\beta_t}{\beta_t}\right] + \beta_t -\tilde{R}_t = 0.
        \end{aligned}
    \end{equation*}
    Combining with Eq.~\eqref{eq:dKL_gap_ScoEst_err} and Eq.~\eqref{ineq:dKL_ineq_mid}, we have 
    \begin{equation}
        \label{ineq:ineq:dKL_ineq_fin}
        \begin{aligned}
            & \frac{\der \KL{q^\gets_t}{\hat{q}_t} }{\der t} \le \sum_{\vy\in \gY} q^\gets_t(\vy) \cdot \left( \sum_{\vy^\prime\not=\vy} R^\gets_t(\vy^\prime,\vy) \cdot \ln\frac{R_t^\gets(\vy^\prime,\vy)}{\tilde{R}_t(\vy^\prime,\vy)} + \tilde{R}_t(\vy) - R^\gets_t(\vy) \right)\\
            &=\sum_{\vy\in \gY} q^\gets_t(\vy) \cdot \left( \sum_{\vy^\prime\not=\vy} R^\gets_t(\vy^\prime,\vy) \cdot \ln\frac{R_t^\gets(\vy^\prime,\vy)}{\tilde{R}_t(\vy^\prime,\vy)} + \sum_{\vy'\neq \vy}\tilde{R}_t(\vy',\vy) -\sum_{\vy'\neq \vy} R^\gets_t(\vy',\vy) \right)\\
            & = \sum_{\vy\in\gY} q^\gets_t(\vy)\cdot\sum_{\vy^\prime\neq \vy} R^\to(\vy,\vy^\prime)\cdot \left[-\frac{q_t^\gets(\vy^\prime)}{q^\gets_t(\vy)} + \hat{v}_{t,\vy}(\vy^\prime)+ \frac{q_t^\gets(\vy^\prime)}{q^\gets_t(\vy)}\ln \frac{q_t^\gets(\vy^\prime)}{q^\gets_t(\vy) \hat{v}_{t,\vy}(\vy^\prime)}\right]\\
            &=\sum_{\vy\in\gY} q^\gets_t(\vy)\cdot\sum_{\vy^\prime\neq \vy} R^\to(\vy,\vy^\prime) \Breg{\frac{q_t^\gets(\vy^\prime)}{q^\gets_t(\vy)}}{\hat{v}_{t,\vy}(\vy^\prime)},
        \end{aligned}
    \end{equation}
    where $D_\phi$ is the Bregman divergence  with $\phi(c)=c\ln c$ (as Eq.~\eqref{eq:score_estimation}), and the last equation follows from the definition of Bregman divergence: 
    \[
    D_\phi(u\Vert v)=\phi(u)-\phi(v)-\langle\nabla \phi(v),u-v\rangle = u \ln \frac{u}{v}-u+v.
    \]
    Then, by Eq.~\eqref{eq:score_estimation} and Assumption~\ref{a4}, we have
    \begin{equation*}
        \int_0^{T-\delta} \der \KL{q^\gets_t}{\hat{q}_t}  \le (T-\delta)\epsilon^2_{\text{score}}.
    \end{equation*}
    Hence, the proof is completed.
\end{proof}

\paragraph{Bounding $\TVD{q_*}{q^\to _{\delta}}$} We adopt the proof strategy of Theorem~6 in \cite{chen2024convergence}.
Consider the forward process $(X_t)_{t\geq0}$. By the coupling characterization of the total variation distance, we have
\[
\TVD{q_*}{q^\to_{\delta}}\coloneqq\inf_{\gamma \in \Gamma(q_*,q^\to_{\delta})} \Pr_{(u,v)\sim \gamma} [u \neq v] \leq \Pr(X_0\neq X_\delta),
\]
where $\Gamma(q_*,q^\to_{\delta})$ is the set of all couplings of $(q_*,q^\to_{\delta})$, and the inequality holds because $(X_0, X_\delta)$ gives a coupling of $(q_*,q^\to_{\delta})$.



By the transition kernel given in \cite[Proposition~3]{chen2024convergence}, we have
$$\Pr(X_0= X_\delta)=\frac{1}{2^{d\log_2 K}}\Pi_{i=1}^{d\log_2K}(1+(-1)^0e^{-2\delta})^{d\log_2 K}=\left(\frac{1+e^{-2\delta}}{2}\right)^{d\log_2 K}\geq e^{-\delta d \log_2 K},$$
where the inequality holds due to the convexity of the exponential function.
Thus,
\begin{equation}
    \label{ineq:early_stopping_tv}
    \TVD{q_*}{q^\to_{\delta}}\leq 1-e^{-\delta d\log_2 K}
\end{equation}

\begin{proof}[Proof of Theorem~\ref{thm:main_thm}]
    We start from the quantization algorithm, i.e., Alg.~\ref{alg:data_quanta}.
    Since the data distribution $p_*$ is supposed to satisfy Assumption~\ref{a1}--\ref{a3}, by introducing Lemma~\ref{lem:discrete_quantization_gap}, the histogram-like approximation $\overline{p}_*$ will be close to $p_*$, i.e.,
    \begin{equation*}
        \TVD{\overline{p}_*}{p_*}\le 3\epsilon
    \end{equation*}
    by choosing
    \begin{equation*}
        L = \sigma\cdot \sqrt{2\ln (2d/\epsilon)} \quad \mathrm{and}\quad  l= \left[2H\cdot \left(\sigma\sqrt{2d\ln(2d/\epsilon)} + d + \sqrt{dm_0}\right)\right]^{-1} \cdot \epsilon.
    \end{equation*}
    Under this condition, we have
    \begin{equation*}
        \begin{aligned}
            &K = \frac{2L}{l} = 4H\cdot \left[2\sigma^2d^{1/2}\cdot  \ln\frac{2d}{\epsilon} + \sigma d\cdot \sqrt{2\ln \frac{2d}{\epsilon} } + d^{1/2}m_0^{1/2}\cdot \sqrt{2\ln \frac{2d}{\epsilon} }\right]\cdot \epsilon^{-1}\\
            & \le 24H\sigma^2dm_0\epsilon^{-1}\cdot \ln(2d/\epsilon)
        \end{aligned}
    \end{equation*}
    where the last inequality follows from $\sigma\ge 1$ and $m_0\ge 1$ without loss of generality. 
    Then, after the training, the implementation of Alg.~\ref{alg:uni_inf} requires $\overline{N}\sim \mathrm{Poisson}(\overline{\beta})$ steps. 
    \paragraph{Proof of bound of the expectation of $\overline{N}$, i.e., $\overline{\beta}$.}
    According to Lemma~\ref{lem:ideal_transtion_event_num}, if we set 
    \begin{equation*}
        t_0 = 0, \quad  t_{w+1} - t_w = 0.5\cdot (T-t_{w+1})\quad \text{and}\quad t_W= T-\delta,
    \end{equation*}
    for the time partitions, 
    \begin{equation*}
        \beta_{t_w}\coloneqq 2d \log_2K/{\min\{1, T-t_w\}}
    \end{equation*}
    for the intermediate Poisson, then it has
    \begin{equation*}
        \begin{aligned}
            & \overline{\beta}= \sum_{k=1}^W \beta_{t_w}\cdot (t_w-t_{w-1}) \le 2d\log_2K\cdot \left(T+\ln(1/\delta)\right)\\
            & \le 2d\cdot \left[\log_2 (24H\sigma^2) + \log_2(dm_0/\epsilon) + \log_2[\ln(2d/\epsilon)] \right]\cdot \left(T+\ln(1/\delta)\right).
        \end{aligned}
    \end{equation*}
    
    \paragraph{Proof of the TV distance bound.} 
    Since our truncated uniformization, i.e., Alg.~\ref{alg:uni_inf}, exactly simulates the reversed process, from Lemma~\ref{thm:unif_linear_convergence}, the KL divergence gap between $q^\gets_{T-\delta} = q^\to_{\delta}$ and $\hat{q}_{T-\delta}$ is bounded by the KL divergence as follows:
    \begin{equation}
        \label{ineq:accumulative_KL}
        \begin{aligned}
            &\KL{q^\gets_{T-\delta}}{\hat{q}_{T-\delta}}\le \KL{q^\gets_{0}}{\hat{q}_{0}}+(T-\delta)\epsilon_{\text{score}}\\
            &\le e^{-T}\cdot d\log_2K + (T-\delta)\epsilon^2_{\text{score}} = \epsilon^2+ (\ln(d/\epsilon)+\ln\log_2K)^2 \cdot \epsilon^2_{\text{score}}\le 2\epsilon^2
        \end{aligned}
    \end{equation}
    where the second inequality follows from Lemma~\ref{lem:fwd_convergence}, the third inequality establishes when $T$ is chosen as
    \begin{equation*}
        T = \ln(d/\epsilon) + \ln\log_2K,
    \end{equation*}
    and the last inequality is established when we have
    \begin{equation*}
        \epsilon_{\text{score}} = \frac{\epsilon}{\ln(d/\epsilon) + \ln\log_2 K} = \tilde{O}(\epsilon).
    \end{equation*}
    Under this condition, due to Pinsker's inequality, Eq.~\eqref{ineq:accumulative_KL} can be relaxed to 
    \begin{equation*}
        \TVD{q_{T-\delta}^\gets}{\hat{q}_{T-\delta}}\le \sqrt{\frac{\KL{q_{T-\delta}^\gets}{\hat{q}_{T-\delta}}}{2}} \le \epsilon.
    \end{equation*}
    Then we have
    \begin{equation*}
        \begin{aligned}
            & \TVD{q_*}{\hat{q}_{T-\delta}}\le \TVD{q_*}{q^\gets_{T-\delta}} + \TVD{q^\gets_{T-\delta}}{\hat{q}_{T-\delta}}\\
            & = \TVD{q_*}{q^\to_\delta} + \TVD{q^\gets_{T-\delta}}{\hat{q}_{T-\delta}} = 1-e^{-\delta d\log_2 K} + \epsilon \le 2\epsilon
        \end{aligned}
    \end{equation*}
    where the second equation follows from Eq.~\eqref{ineq:early_stopping_tv} and the last inequality is established by requiring
    \begin{equation*}
        \delta\le \frac{\epsilon}{d\cdot \log_2K}\quad \Leftrightarrow\quad \delta d\log_2K\le \epsilon.
    \end{equation*}
    Under this condition, we have
    \begin{equation*}
        \delta d\log_2K\le \epsilon\le \ln \frac{1}{1-\epsilon} \quad \Rightarrow\quad 1-e^{-\delta d\log_2 K}\le \epsilon.
    \end{equation*}
    Suppose the underlying distributions of $\overline{\rvy}, \hat{\rvx}$ are $\overline{q}, \hat{p}$ respectively, due to the connection between $\hat{p}, \overline{p}_*$ and $\overline{q}, \overline{q}_*$ shown in Eq.~\eqref{def:quantized_mass_func}, we have
    \begin{equation*}
        \TVD{\overline{p}_*}{\hat{p}} = \int \left|\overline{p}_*(\vx) - \hat{p}(\vx)\right| \der \vx =\sum_{\overline{\vy}\in \overline{\gY}}\left|\overline{q}(\overline{\vy}) - \overline{q}_*(\overline{\vy})\right| = \sum_{\vy\in\gY} \left|\hat{q}_{T-\delta}(\vy) - \hat{q}_*(\vy)\right|\le 2\epsilon.
    \end{equation*}
    Combining this result with Lemma~\ref{lem:discrete_quantization_gap}, we have
    \begin{equation*}
        \TVD{p_*}{\hat{p}}\le \TVD{p_*}{\overline{p}_*} + \TVD{\overline{p}_*}{\hat{p}} \le 3\epsilon + 2\epsilon \le 5\epsilon.
    \end{equation*}
    Hence, the proof is completed.
\end{proof}

\end{document}